\documentclass[twoside,11pt]{article}

\usepackage[preprint,nohyperref]{jmlr2e}

\usepackage{times}

\usepackage{amsmath,amsfonts,bm}

\def\figref#1{figure~\ref{#1}}
\def\Figref#1{Figure~\ref{#1}}

\def\secref#1{section~\ref{#1}}
\def\Secref#1{Section~\ref{#1}}

\def\eqref#1{equation~\ref{#1}}
\def\Eqref#1{Equation~\ref{#1}}

\def\1{\bm{1}}

\DeclareMathAlphabet{\mathsfit}{\encodingdefault}{\sfdefault}{m}{sl}
\SetMathAlphabet{\mathsfit}{bold}{\encodingdefault}{\sfdefault}{bx}{n}

\DeclareMathOperator*{\argmax}{arg\,max}

\usepackage{graphicx}
\usepackage{subcaption}
\usepackage[noend]{algpseudocode}
\usepackage{algorithm}
\usepackage{amsmath, amsthm, amssymb, amsfonts, physics}
\usepackage{xspace}

\usepackage{dsfont}
\usepackage{multicol}
\usepackage{wrapfig}

\usepackage[utf8]{inputenc}
\usepackage{latexsym}

\usepackage{tikz}
\usetikzlibrary{decorations.pathreplacing}
\usetikzlibrary{fadings}

\usepackage{etoolbox}
\usepackage{indentfirst}

\makeatletter

\usepackage{mathtools}
\usepackage{natbib}
\setcitestyle{round}
\usepackage{array}
\usepackage{mathrsfs}
\usepackage{comment}
\usepackage{enumitem}
\setlist[itemize]{leftmargin=1.5em}
\setlist[enumerate]{leftmargin=1.5em}

\usepackage{algorithm}
\usepackage{algpseudocode}
\usepackage{lineno}
\usepackage{arydshln}
\usepackage{graphicx}

\usepackage{caption}
\usepackage{subcaption}

\newcommand{\selffigref}[1]{(\ref{#1})}

\newcommand{\appref}[1]{Appendix \ref{#1}}

\newcommand{\propref}[1]{Proposition~\ref{#1}}
\newcommand{\lemref}[1]{Lemma~\ref{#1}}

\newcommand{\algoref}[1]{Algorithm~\ref{#1}}

\newcommand{\paren} [1] {\ensuremath{ \left( {#1} \right) }}

\newcommand{\nats}{\ensuremath{\mathbb{N}}}
\newcommand{\reals}{\ensuremath{\mathbb{R}}}

\renewcommand{\Pr}[1]{\ensuremath{\mathbb{P}\left[#1\right] }}

\newcommand{\mutualinfo}[1]{\mathbb{I}\paren{#1}}

\newcommand{\denselist}{\itemsep 0pt\topsep-6pt\partopsep-6pt}

\newcommand{\bh}{{\mathbf{h}}}

\newcommand{\bx}{{\mathbf{x}}}

\newcommand{\by}{{\mathbf{y}}}

\usepackage{ifthen}

\newboolean{showcomments}
\setboolean{showcomments}{true}

\newcommand{\algname}{\textsf{Algname}\xspace}
\newcommand{\tempParam}{\zeta}
\newcommand{\regWeight}{\rho}
\newcommand{\regParam}{\lambda}
\newcommand{\indexRVar}{\phi}
\newcommand{\maxInfo}{\gamma}
\newcommand{\algParam}{\theta}

\newcommand{\yuxin}[1]{\ifthenelse{\boolean{showcomments}}{\textcolor{blue}{[{\bf Yuxin:} #1]}{}}}
\newcommand{\fengxue}[1]{\ifthenelse{\boolean{showcomments}}{\textcolor{brown}{ [{\bf Fengxue:} #1]}{}}}
\newcommand{\brian}[1]{\ifthenelse{\boolean{showcomments}}{ \textcolor{purple}{[{\bf Brian:} #1]}{}}}
\newcommand{\todo}[1]{\ifthenelse{\boolean{showcomments}}{\textcolor{red}{[{\bf TODO:} #1]}{}}}

\renewcommand{\algname}{\textsc{LOCo}\xspace}
\newcommand{\LOCo}{\textsc{LOCo}\xspace}

\newcommand{\UCB}{\textsc{UCB}\xspace}
\newcommand{\TS}{\textsc{TS}\xspace}

\newcommand{\SE}{\textsc{SE}\xspace}
\newcommand{\col}{\textrm{col}\xspace}

\newcommand{\rebuttal}[1]{ #1}

\newcommand{\instance}[0]{\ensuremath{\mathbf{x}}}
\newcommand{\latentrep}[0]{\ensuremath{\mathbf{z}}}
\newcommand{\GramMat}[0]{\ensuremath{\mathbf{K}}}
\newcommand{\Selected}[0]{\ensuremath{\mathbf{A}}}
\newcommand{\DataSet}[0]{\ensuremath{\mathbf{D}}}
\newcommand{\LatentRepSet}[0]{\ensuremath{\mathbf{Z}}}

\definecolor{mydarkblue}{rgb}{0,0.08,0.45}
\definecolor{darkgreen}{RGB}{34,139,34}

\definecolor{myblue}{RGB}{49,130,189}
\definecolor{myred}{RGB}{251,106,74}

\title{Learning Representation for Bayesian Optimization with Collision-free Regularization}

\author{\name Fengxue Zhang \email zhangfx@uchicago.edu \\
       \addr University of Chicago\\
       \AND
       \name Brian Nord \email nord@fnal.gov \\
       \addr Fermi National Accelerator Laboratory \\
       \AND
       \name Yuxin Chen \email chenyuxin@uchicago.edu \\
       \addr University of Chicago\\
       }
\editor{x}

\begin{document}

\maketitle
\begin{abstract}

Bayesian optimization has been challenged by datasets with large-scale, high-dimensional, and non-stationary characteristics, which are common in real-world scenarios. Recent works attempt to handle such input by applying neural networks ahead of the classical Gaussian process to learn a latent representation. We show that even with proper network design, such learned representation often leads to collision in the latent space: two points with significantly different observations collide in the learned latent space, leading to degraded optimization performance. To address this issue, we propose \LOCo, an efficient deep Bayesian optimization framework which employs a novel regularizer to reduce the collision in the learned latent space and encourage the mapping from the latent space to the objective value to be Lipschitz continuous. \LOCo takes in pairs of data points and penalizes those too close in the latent space compared to their target space distance. We provide a rigorous theoretical justification for \LOCo by inspecting the regret of this dynamic-embedding-based Bayesian optimization algorithm, where the neural network is iteratively retrained with the regularizer. Our empirical results demonstrate the effectiveness of \LOCo on several synthetic and real-world benchmark Bayesian optimization tasks.
\end{abstract}


\section{Introduction}
Bayesian optimization is a classical sequential optimization method and is widely used in various fields in science and engineering, including recommender systems \citep{galuzzi2019bayesian}, medical trials \citep{sui2018stagewise}, robotic controller optimization \citep{berkenkamp2016safe}, scientific experimental design \citep{yang2019machine}, and hyper-parameter tuning \citep{snoek2012practical}, among many others. Many of these applications involve evaluating an expensive blackbox function; therefore, the number of queries should be minimized. A common way to model the unknown function is via Gaussian processes (GPs) \citep{rasmussen:williams:2006}. GP has been extensively studied under the bandit setting, as an effective surrogate model 
in a broad class of blackbox function optimization problems \citep{10.5555/3104322.3104451,Djolonga2013HighDimensionalGP}.

A key challenge for learning with GPs lies in designing and optimizing kernels used for modeling the covariance structures. Such an optimization task depends on both the prior knowledge of the input space and the dimension of the input space. 
{For structural or high-dimensional data, it is often prohibitive to design and test a GP model. Specifically, local kernel machines are known to suffer from the curse of dimensionality \citep{bengio2005curse}, while the required number of training samples could grow exponentially with the dimensionality of the data. Therefore, 
representation learning is needed to optimize the learning process.}

\begin{figure*}[!t]
     \centering
     \begin{subfigure}[b]{0.24\textwidth}
        \captionsetup{justification=centering}
         \centering
         \includegraphics[trim={0pt 10pt 0pt 0pt},
         width=\textwidth]{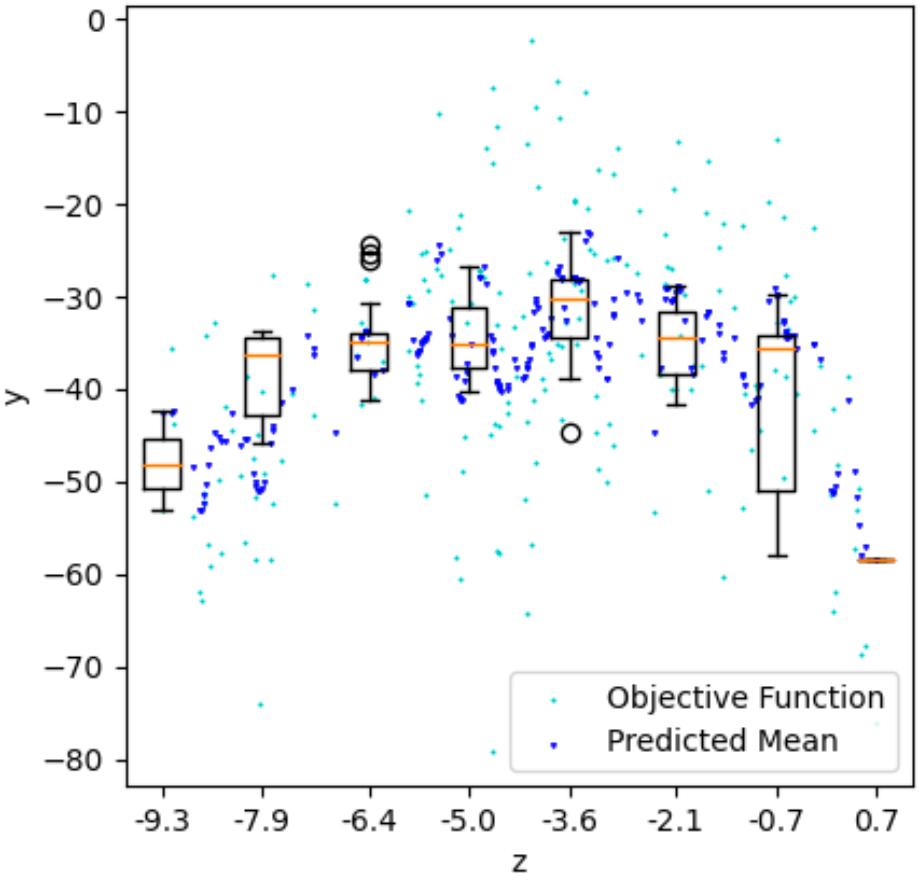}
         \caption{\scriptsize Latent space\\ w/o regularization}
         \label{fig:collision_latent}
     \end{subfigure}
      \begin{subfigure}[b]{0.24\textwidth}
        \captionsetup{justification=centering}
         \centering
         \includegraphics[trim={0pt 10pt 0pt 0pt},
         width=\textwidth]{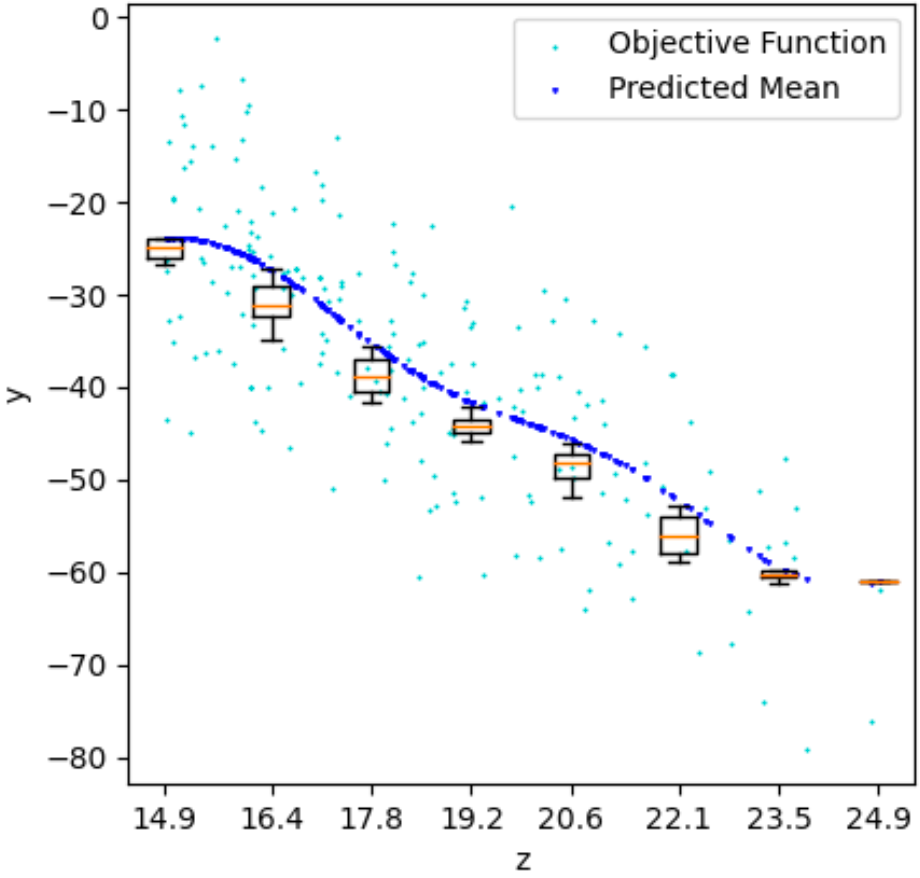}
         \caption{\scriptsize Latent space\\ w/ regularization}
         \label{fig:reg_latent}
     \end{subfigure}
     \begin{subfigure}[b]{0.24\textwidth}
        \captionsetup{justification=centering}
         \centering
         \includegraphics[trim={0pt 10pt 0pt 0pt},
         width=\textwidth]{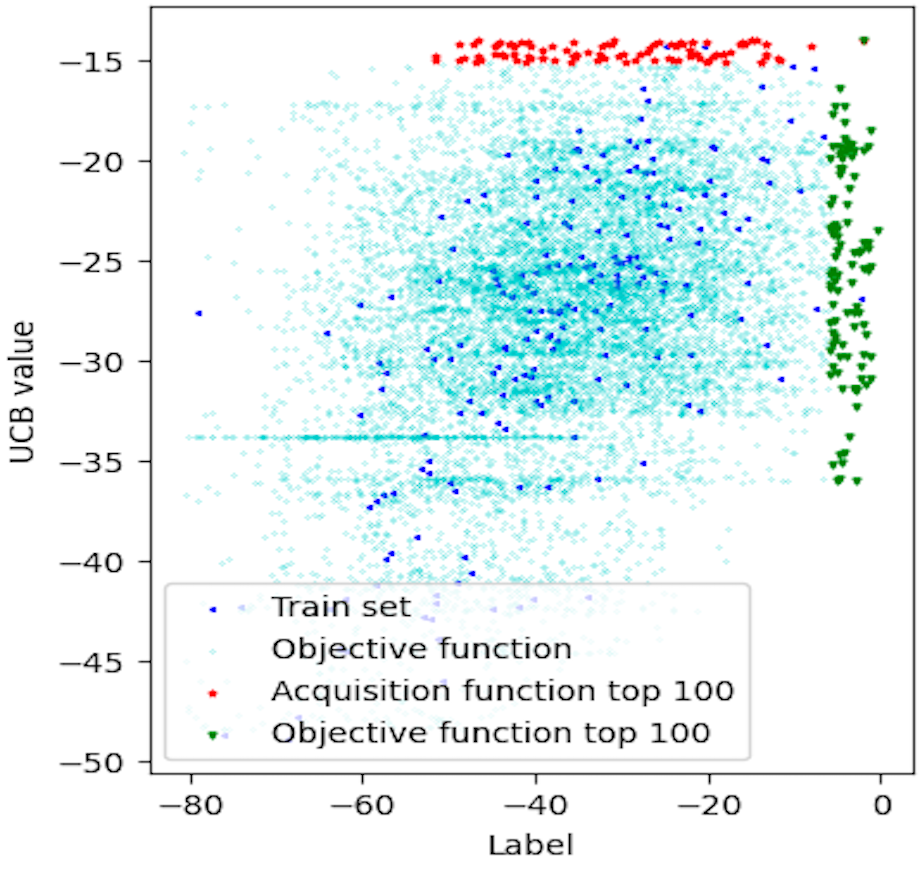}
         \caption{\scriptsize Acq. function\\ w/o regularization}
         \label{fig:collision_acq}
     \end{subfigure}
     \begin{subfigure}[b]{0.24\textwidth}
        \captionsetup{justification=centering}
         \centering
         \includegraphics[trim={0pt 10pt 0pt 0pt},
         width=\textwidth]{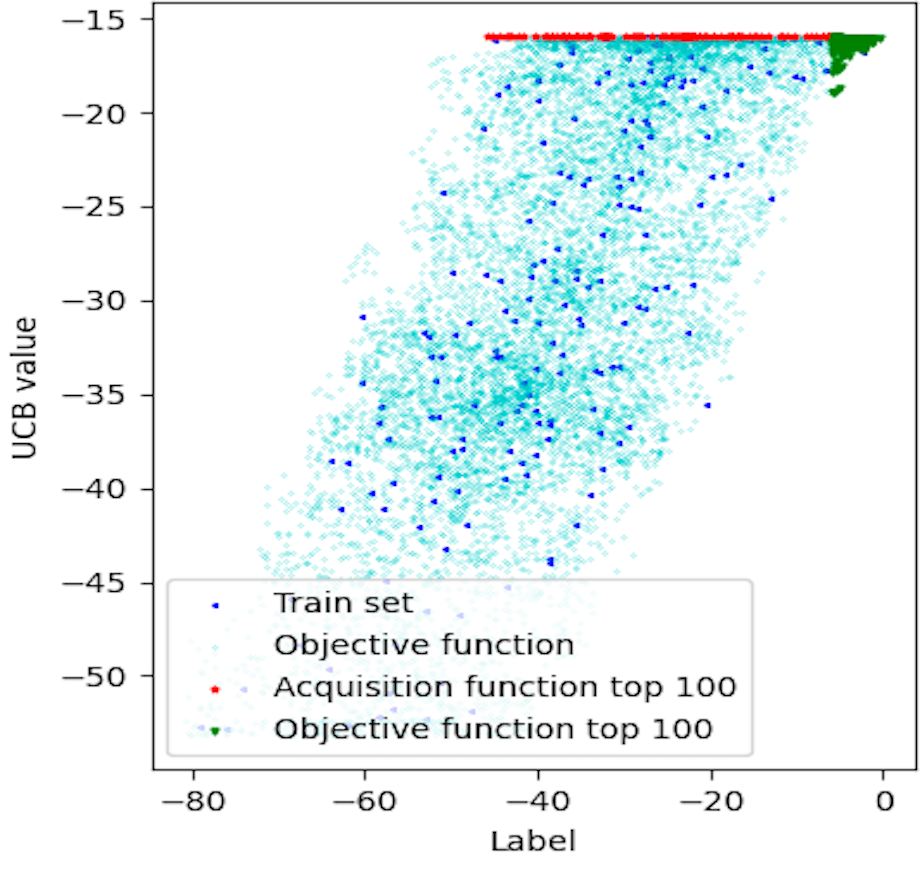}
         \caption{\scriptsize Acq. function\\ w/ regularization}
         \label{fig:reg_acq}
     \end{subfigure}

    \caption{\selffigref{fig:collision_latent} and \selffigref{fig:reg_latent} shows the collision of a 100-point training set in the latent space of Rastrigin-2D function \citep{10018403158} which will be discussed in detail in \secref{sec:experiments}. 
    The box plot shows the distribution of the predicted mean of the training set at certain positions in the latent space. Specifically, each box at a certain tick on the z-axis corresponds to the distribution of predicted values in the interval between its tick and the right next tick on the z-axis. \selffigref{fig:collision_acq} and \selffigref{fig:reg_acq} shows the corresponding acquisition function distribution against the label for both the training set and the whole objective function.
    \selffigref{fig:reg_latent} and \selffigref{fig:reg_acq} are learned with the proposed collision mitigation algorithm. 
    Note that despite the difference in the of the acquisition function values, the training negative log likelihood (i.e. training loss) and the training set Mean Sqaure Error (MSE) for the two training results are close. 
    Compared to \selffigref{fig:reg_latent}, \selffigref{fig:collision_latent} bears more collision which introduces the heterogeneous noise. The additional noise affects the acquisition function as is shown in the comparison between \selffigref{fig:collision_acq} and \selffigref{fig:reg_acq}. The resulted larger green and red areas in \selffigref{fig:collision_acq} reflects the distraction caused by the collision.
    }
    \label{fig:four_graphs}
    \vspace{-4mm}
\end{figure*}

Recently, Gaussian process optimization has been investigated in the context of latent space models. For example, deep kernel learning \citep{pmlr-v51-wilson16}  learns a latent 
data representation and a scalable kernel simultaneously via an end-to-end trainable deep neural network. In general, the neural network is trained to learn a simpler latent representation with reduced dimension and has the structure information already embedded for the GP. Combining the representation learned via a neural network with GP could improve the scalability and extensibility of classical Bayesian optimization, but it also poses new challenges for the optimization task, such as dealing with the tradeoff between representation learning and function optimization \citep{Tripp2020SampleEfficientOI}.

As we later demonstrate, a critical challenge brought by representation learning in Bayesian optimization is that the latent representation is prone to \emph{collisions}: two points with significantly different observations can get too close, and therefore collide in the latent space. The collision effect in latent space models for Bayesian optimization is especially evident when information is lost during dimensionality reduction and/or when the training data is limited in size.

As illustrated in \Figref{fig:four_graphs}, when passed through the neural network, data points with drastically different observations are mapped to close positions in the latent space (see \figref{fig:collision_latent}). Such collisions could be regarded as additional heterogeneous noise introduced by the neural network. Although Bayesian optimization is known to be robust to mild noisy observations \citep{bogunovic2018adversarially}, the collision in latent space could be harmful to the optimization performance, as it is non-trivial to model the collision into the acquisition function explicitly. Also, the additional noise induced by the collision effect will further loosen the regret bound for classical Bayesian optimization algorithms \citep{10.5555/3104322.3104451}. 
The similar training loss and mean squared error between the examples bearing collision and the example trained on the same dataset with the proposed collision mitigation method indicates that improving learning loss could not guarantee to reduce the collision.

\paragraph{Overview of main results}


To mitigate the collision effect, we propose a novel regularization scheme that can be applied as a simple plugin amendment for the latent space based Bayesian optimization models. The proposed algorithm, namely \textit{\underline{L}atent Space \underline{O}ptimization via \underline{Co}llision-free regularization}  (\algname), leverages a regularized regression loss function to optimize the latent space for Bayesian optimization periodically.

Concretely, our collision-free regularizer is encoded by a novel \emph{pairwise collision penalty} function defined jointly on the latent space and the output domain. To mitigate the risk of collision in the latent space (and consequently boost the optimization performance), \algname applies the regularizer to minimize the collisions uniformly in the latent space.

We further note that for Bayesian optimization tasks, collisions in regions close to the optimum are more likely to mislead the optimization algorithm.
Based on this insight, we propose an optimization-aware regularization scheme that assigns higher weight to the collision penalty on those pairs of points closer to the optimum region in the latent space. This algorithm, which we refer to as \textit{\underline{D}ynamically-\underline{W}eighted \algname} (DW \algname), is designed to  dynamically assess the importance of a collision during optimization. Compared with the uniform collision penalty in the latent space, the dynamic weighting mechanism has demonstrated drastic improvement over the state-of-the-art latent space based Bayesian optimization models.

We summarize our key contributions as follows:
\begin{enumerate}[label={\Roman*.},leftmargin=*]\denselist
    \item We investigate latent space based Bayesian optimization, and expose the limitations of existing latent space optimization approaches due to the collision effect on the latent space (\Secref{sec:statement}).
    \item We propose a novel regularization scheme as a simple plugin amendment for latent-space-based Bayesian optimization models. Our regularizer penalizes collisions in the latent space and effectively reduces the collision effect. Furthermore, we propose an optimization-aware dynamic weighting mechanism for adjusting the collision penalty to improve the effectiveness of regularization for Bayesian optimization (\Secref{sec:method}).
    \item We provide theoretical analysis for the performance of Bayesian optimization on regularized latent space (\Secref{sec:analysis}).
    \item We conducted an extensive empirical study on several synthetic and real-world datasets, including a real-world case study for cosmic experimental design, and demonstrate the promising empirical performance for our algorithm (\Secref{sec:experiments}).
\end{enumerate}

\vspace{-3mm}
\section{Related Work}\label{sec:related}
\vspace{-2mm}


This section provides a short survey on recent work in Bayesian learning, which was designed to overcome the kernel design challenge for Gaussian process regression tasks and Bayesian optimization.

\paragraph{Different surrogate models with internal latent space}
Some alternative surrogate models have been proposed to replace classical kernel-based GP in Bayesian optimization to overcome the challenge of high-dimensional and highly-structured input in BO.
\textit{Deep Network for Global Optimization} 
\citep{snoek2015scalable} uses a pre-trained deep neural network with a Bayesian linear regressor at the last hidden layer of the network as the surrogate model.
More generally, \textit{Deep Kernel Learning} (\textsc{DKL}) combines the power of the Gaussian process and neural network by introducing a deep neural network $g$ to learn a mapping $g:\mathcal{X} \rightarrow \mathcal{Z}$ from the input domain $\mathcal{X}$ to a latent space $\mathcal{Z}$ \citep{pmlr-v51-wilson16}. It uses the latent representation $\latentrep \in \mathcal{Z}$ as the input of the base GP. The neural network $g$ and a spectral mixture-based kernel $k$ form a scalable expressive closed-form deep covariance kernel, denoted by $k_{\text{DK}}(\instance_i,\instance_j) \rightarrow k(g(\instance_i),g(\instance_j))$. The deep kernel allows end-to-end learning and Bayesian optimization on the original input space.  

Recently, \citet{van2021feature} and \citet{ober2021promises} studied the pitfalls of \textsc{DKL} in terms of \textit{feature collapse} and overfitting, and proposed to tackle the problem with either bi-Lipschitz constraints or applying stochastic gradient Langevin dynamics \citep{welling2011bayesian}
The feature collapse problem is related to the collision effect studied in this paper, where distinct points are folded into the same location in the latent space. The key difference is that feature collapse considers the collapses of data points distinct in their \textit{input representation}; in contrast, we focus on the collision effect harmful for the \textit{optimization task}, and focus on collisions where the folded points correspond to drastically different \textit{labels} as it is even desirable in optimization that distant points mapped to the same position in the latent space as long as the corresponding labels are close.

\paragraph{Representation learning and latent space optimization}

Instead of reducing the dimensionality and performing optimization in an end-to-end process, other methods aim to optimize in a related latent space first and then map the solution back to the original input space. \citet{Djolonga2013HighDimensionalGP} assume that only a subset of input dimensions varies, and the kernel is smooth (i.e. with bounded RKHS norm). Under these assumptions, the underlying subspace is learned via low-rank matrix recovery. \emph{Random feature} is another solution under this setting \citep{rahimi2007random,letham2020re,binois2015warped,nayebi2019framework,wang2016bayesian}. It is known that a random representation space of sufficiently large dimension is guaranteed to contain the optima with high probability. 
\cite{mutny2019efficient} consider \emph{Quadrature Fourier Features} (QFF)---as opposed to \emph{Random Fourier Feature} (RFF) in \citet{rahimi2007random}---to overcome the variance starvation problem, and proved that Thompson sampling and GP-UCB achieve no-regret with squared exponential kernel in optimization tasks. However, both RFF and QFF methods rely on a key assumption that the function to be optimized has a low \emph{effective dimension}. In contrast, as discussed in \Secref{sec:experiments} and the supplemental materials, we show that \algname performs well for challenging high-dimensional BO problems where algorithms relying on the low effective dimension assumption may fail.

Another line of work on latent space optimization uses \textit{autoencoders} to learn latent representations of the inputs to improve the scalability and capability to leverage the structural information \citep{mathieu2019disentangling}, \citep{ding2020guided}, \citep{gomez2018automatic,10.5555/2832747.2832748,Tripp2020SampleEfficientOI,pmlr-v80-lu18c}. \citet{mathieu2019disentangling}, \citet{ding2020guided} focus on disentangled representation learning that breaks down, or disentangles, each feature into narrowly defined variables and encodes them as separate dimensions.\citet{Tripp2020SampleEfficientOI} iteratively train the autoencoder with a dynamic weighting scheme when performing optimization to improve the embedding.
\citet{griffiths2020constrained} and \citet{letham2020re} enforce certain properties on the representation space to improve the optimization performance. To the best of the authors' knowledge, collision of the embeddings has not been explicitly studied. \citet{binois2015warped} propose a \textit{warped kernel} to guarantee the injectivity in the random linear embedding, which is not applicable in neural network-based methods. 
\citet{grosnit2021high} use VAE and contrastive learning which similarly encourages latent space separation. They rely on categorical output to define the learning loss for VAE model while we don't have such constraints on the output.

A common challenge in applying these techniques to generic optimization tasks lies in the assumption on the accessibility of training data: Bayesian optimization often assumes limited access to labeled data, while surrogate models built on deep neural networks often rely on abundant access to data for pretraining. Another problem lies in the training objective: During training, these surrogate models typically focus on improving the \emph{regression} performance, and do not explicitly address the artifact caused by collisions of the learned embeddings, which---as shown in \Secref{sec:collision}---could be harmful to sequential decision-making tasks.


\section{Problem Statement}\label{sec:statement}
In this section, we introduce necessary notations and formally state the problem. We focus on the problem of sequentially optimizing a function $f: \mathcal{X}\rightarrow \reals$, where $\mathcal{X} \subseteq \reals^d$ is the input domain. At iteration $t$, we pick a point $\instance_t\in\mathcal{X}$, and observe the function value perturbed by additive noise: $y_t = f(\instance_t) + \epsilon_t$ with $\epsilon_t \sim \mathcal{N}(0, \sigma^2)$ being i.i.d. Gaussian noise. Our goal is to maximize the sum of rewards $\sum^T_{t=1}f(\instance_t)$ over $T$ iterations, or equivalently, to minimize the \emph{cumulative regret} $R_T := \sum_{t=1}^T r_t$, where $r_t := \max\limits_{\instance\in \mathcal{X}}f(\instance) - f(\instance_t)$ denotes the \textit{instantaneous regret}. 
We also consider another common performance metric in BO, i.e. the simple regret $r^*_T = \max\limits_{\instance\in \mathcal{X}}f(\instance) - \max\limits_{t\leq T}f(\instance_t)$. 

\subsection{Bayesian Optimization}

Formally, we assume that the underlying function $f$ is drawn from a Gaussian process, denoted by $\mathcal{GP}(m(\instance), k(\instance, \instance'))$, where $m(\instance)$ is the mean function and $k(\instance, \instance')$ is the covariance function. At iteration $t$, given the selected points $\Selected_t=\{\instance_1, \dots ,\instance_t\}$ and the corresponding noisy evaluations $\mathbf{y}_t=[y_1,\dots, y_t]^\top$, the posterior over $f$ also takes the form of a GP, with mean $\mu_t(\instance) = k_t(\instance)^\top(\GramMat_t+\sigma^2I)^{-1}\by_t$ and covariance $k_t(\instance, \instance') = k(\instance,\instance')-k_t(\instance)^\top(\GramMat_t+\sigma^2I)^{-1}k_t(\instance')$, 
where $k_t(\instance) = [k(\instance_1, \instance),\dots, k(\instance_t, \instance)]^\top$ and $\GramMat_t := [k(\instance, \instance')]_{\instance,\instance' \in \Selected_t}$ is the positive definite kernel matrix \citep{10.5555/1162254}. After obtaining the posterior, one can compute the acquisition function $\alpha: \mathcal{X} \rightarrow \reals$, which is used to select the next point to be evaluated. Various acquisition functions have been proposed in the literature, including popular choices such as Upper Confidence Bound (\UCB) \citep{10.5555/3104322.3104451} and Thompson sampling (\TS) \citep{10.1093/biomet/25.3-4.285}. 

\subsection{Latent Space Optimization}
Recently, \textit{Latent Space Optimization} (LSO) has been proposed to solve Bayesian optimization problems on complex input domains \citep{gomez2018automatic,10.5555/2832747.2832748,Tripp2020SampleEfficientOI,pmlr-v80-lu18c}. LSO learns a latent space mapping $g: \mathcal{X}\rightarrow \mathcal{Z}$  to convert the input space $\mathcal{X}$ to the latent space $\mathcal{Z}$. Then, it constructs an objective mapping $h: \mathcal{Z} \rightarrow \reals$ such that $f(\instance)\approx h(g(\instance)),\ \forall \latentrep\in \mathcal{Z}$. In this paper, we model the latent space mapping $g$ as a neural network;
the neural network $g$ and the base kernel $k$ together are regarded as a \emph{deep kernel}, denote by $k_{\text{nn}}(\instance, \instance') = k(g(\instance), g(\instance'))$ \citep{pmlr-v51-wilson16}. In this context, the actual input space for BO is the latent space $\mathcal{Z}$ and the objective function is $h$. With the acquisition function $\alpha_{\text{nn}}(\instance) := \alpha(g(\instance))$, we do not compute an inverse mapping $g^{-1}$ as opposed to the aforementioned autoencoder-based LSO algorithms (e.g. \citet{Tripp2020SampleEfficientOI}), since BO directly select $\instance_t = \argmax\limits_{\instance\in \mathcal{X}} \alpha_{\text{nn}}(\instance)\ \forall t\leq T$. 
In our analysis, we use squared exponential kernel, i.e. \rebuttal{$k_{\SE}(\instance,\instance')=\sigma^2_{\SE}\exp(-\frac{(\instance-\instance')^2}{2l})$.}


\subsection{The Collision Effect of LSO} \label{sec:collision}
When the mapping $g: \mathcal{X} \rightarrow \mathcal{Z}$ is represented by a neural network, it may cause undesirable \emph{collisions} between different input points in the latent space $\mathcal{Z}$. Under the noise-free setting, we say there exists a \emph{collision} in $\mathcal{Z}$, if $\exists \instance_i, \instance_j \in \mathcal{X}$, such that when $g(\instance_i) = g(\instance_j)$, $|f(\instance_i) - f(\instance_j)| > 0$. Such collision could be regarded as additional (unknown) noise on the observations introduced by the neural network $g$. 
Given a representation function $g$, noisy observations $y = f(\instance) + \epsilon$, we say that there exists a collision, if for $\lambda > 0$, there exist $\instance_i, \instance_j \in \mathcal{X}$, such that $|g(\instance_i) - g(\instance_j)| < \lambda |y_i - y_j|$. 

When the distance between a pair of points $(\instance_i, \instance_j)$ in the latent space is too close compared to their difference in the output space, the different output values $y_i, y_j$ for the collided points in the latent space could be interpreted as the effect of additional observation noise for $g(\instance_i)$ (or $g(\instance_j)$). In general, collisions could degrade the performance of LSO. Since the collision effect is \emph{a priori} unknown, it is often challenging to be dealt with in LSO, even if we regard it as additional observation noise and increase the (default) noise variance in the Gaussian process. Thus, it is necessary to mitigate the collision effect by directly restraining it in the representation learning phase. One potential method to avoid collision could be tuning the design of neural networks. However, we empirically show that increasing the network complexity often does \textit{not} help to reduce the collision. \rebuttal{The study is posed in \appref{app:collision_vs_nn}}.

\rebuttal{We consider a low-noise setting where the collision can play a more significant role in degrading the optimization performance. As is shown in \Figref{fig:four_graphs} that the collision could result in larger difficulty in the optimization task. And when a collision exists, it is hard to distinguish it from the observation. Therefore we focused on treating the collision when defining the penalty instead of dealing with the stochasticity.}

\section{Latent Space Optimization via Collision-free Regularization}\label{sec:method}
We now introduce \algname 
as a novel algorithmic framework to mitigate the collision effect.

\subsection{Overview of the \algname Algorithm}
The major challenge in restraining collisions in the latent space is that---unlike the formulation of the classical regression loss---we cannot quantify it based on a single training example. We can, however, quantify collisions by grouping pairs of data points and inspecting their corresponding observations.

\begin{figure}[!ht]
  \begin{center}
    \includegraphics[width=.65\textwidth]{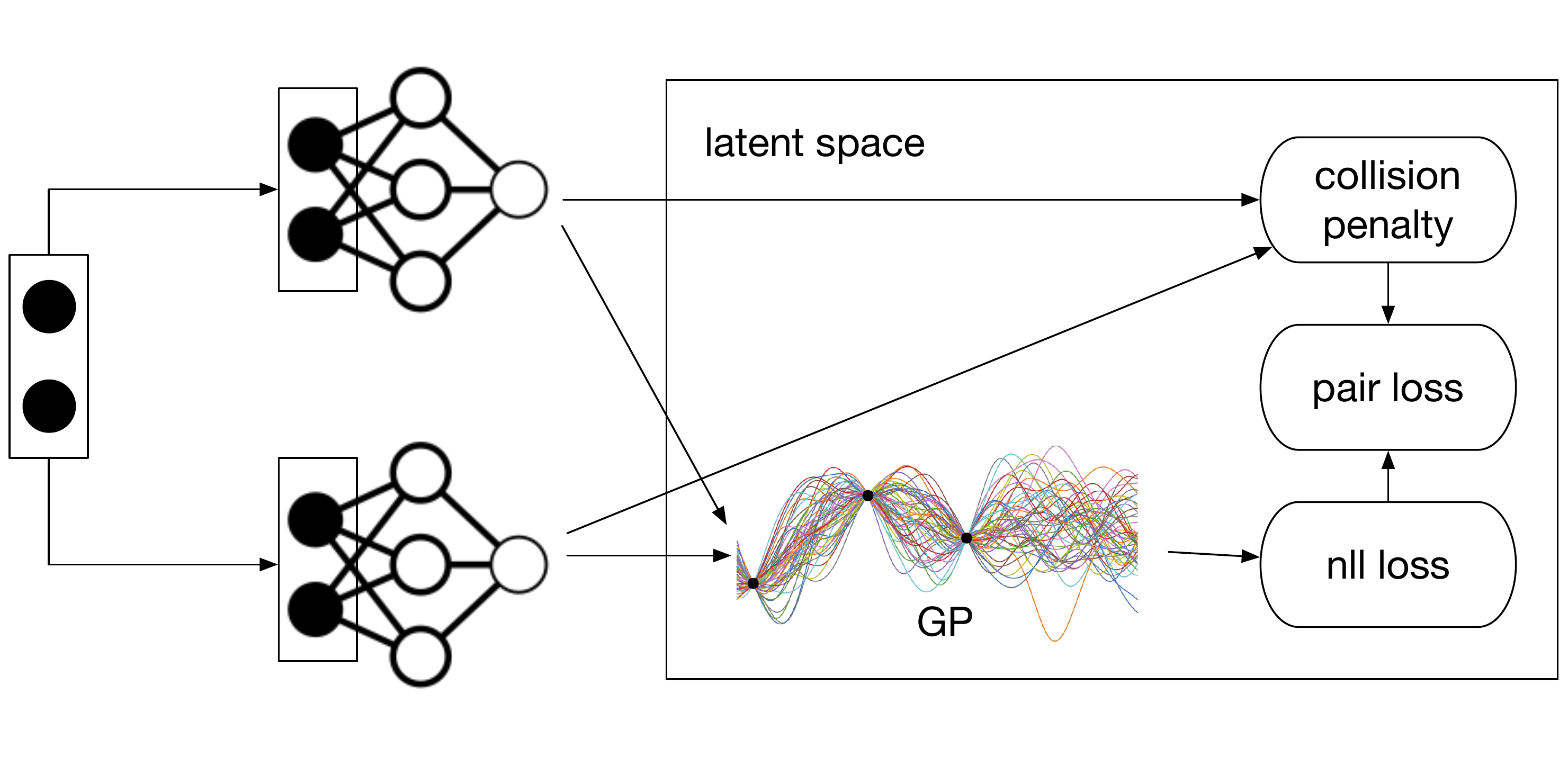}
\end{center}
\caption{Schematic of \algname} \label{fig:coflo-sketch}
\vspace{-2mm}
\end{figure}

We define the \emph{collision penalty} based on pairs of inputs and further introduce a pair loss function to characterize the collision effect.
Based on this pair loss, we propose a novel regularized latent space optimization algorithm\footnote{Note that we have introduced several hyper-parameters in the algorithm design; we will defer our discussion on the choice of these parameters to \Secref{sec:experiments}.}, as summarized in \algoref{alg:main}. The proposed algorithm concurrently feeds the pair-wise input into the same network and calculates the pair loss function. We demonstrate this process in \Figref{fig:coflo-sketch}.

Given a set of labeled data points, we can train the neural network to create an initial latent space representation similar to DKL \citep{pmlr-v51-wilson16}\footnote{To obtain an initial embedding in the latent space, the process does not require the labels to be exact and allows the labels to be collected from a related task of cheaper cost.}. Once provided with the initial representation, we can then refine the latent space by running \algname and periodically update the latent space (i.e. updating the learned representation after collecting a batch of data points) to mitigate the collision effect as we gather more labels.


\begin{algorithm}
  \caption{\textbf{\underline{L}}atent Space \textbf{\underline{O}}ptimization via \textbf{\underline{Co}}llision-free Regularization  (\algname)}\label{alg:main}
  \begin{algorithmic}[1]
    \State {\bf Input}: Penalty parameter $\regParam$ (cf. \Eqref{eq:penalty}), regularization weight $\regWeight$ (cf. \Eqref{eq:pairloss}), 
    importance weight parameter $\tempParam$ (cf. \Eqref{eq:weighted}), neural network $g$, parameters $\algParam_t = (\algParam_{h,t}, \algParam_{g,t})$, total time steps $T$;
    \For{$t = 1\ to\ T$}
        \State $\instance_t \leftarrow \argmax\limits_{\instance\in D}{\alpha(g(\instance, \algParam_{g, t}))}$ \Comment{\textit{acquire}}
        \State $y_t \leftarrow$ evaluation on $\instance_t$ \Comment{\textit{update observation}}
            \State $\algParam_{t+1} \leftarrow$ retrain $\algParam_t$ with the pair loss function 
            $L_{\regWeight,\regParam,\tempParam, g}(\algParam_t, D_t)$ as in \Eqref{eq:pairloss}
    \EndFor
    \State {\bf Output}: $\max\limits_{t}{y_t}$
  \end{algorithmic}
\end{algorithm}

\vspace{-3mm}

\subsection{Collision Penalty}

This subsection aims to quantify the collision effect based on the definition proposed in \Secref{sec:collision}. As illustrated in \Figref{fig:coflo-sketch}, we feed pairs of data points into the neural network and obtain their latent space representations. Apart from maximizing the GP's likelihood, we concurrently calculate the amount of collision on each pair and incur a penalty when the value is positive.
For $\instance_i, \instance_j \in \mathcal{X}$, $y_i = f(\instance_i) + \epsilon$, \rebuttal{$y_j = f(\instance_j) + \epsilon$} are the corresponding observations, and $\latentrep_i = g(\instance_i)$, $\latentrep_j = g(\instance_j)$ are the corresponding latent space representations. We define the \emph{collision penalty} as
\begin{align}
    p_{ij} = \max(\regParam|y_i - y_j| - |\latentrep_i - \latentrep_j|, 0)\label{eq:penalty}
\end{align}
where $\regParam$ is a penalty parameter that controls the smoothness of the target function $h: \mathcal{Z} \rightarrow \reals$. 
\rebuttal{As a rule of thumb, one can estimate $\regParam$ by sampling from the original data distribution $\mathbb{P}(X,Y)$, $(X,Y)\in \mathcal{X}\times \reals$, and choose the $\lambda$ to be the maximum value such that $\sum_{i,j}\max(\regParam|y_i - y_j| - |\instance_i - \instance_j|, 0) = 0$ (i.e. to provide an upper bound for $\lambda$ by keeping the total collision in the input domain to be zero).}

\subsection{Importance-Weighted Collision-Free Regularizer}\label{sec:dynamicweight}
\rebuttal{Note that it is challenging to universally reduce the collisions by minimizing the collision penalty and the GP's regression loss}---this is particularly the case with a limited amount of training data. Fortunately, for optimization tasks, it is often unnecessary to learn fine-grained representation for suboptimal regions. {Therefore, we can dedicate more training resources to improve the learned latent space pertaining to the potentially near-optimal regions}. Following this insight, we propose to use a weighted collision penalty function, which uses the objective values for each pair as an importance weight in each iteration.
Formally, for any pair $((\instance_j, \latentrep_j, y_j),(\instance_i, \latentrep_i, y_i))$ in a batch of observation pairs $D_t = \{((\instance_m, \latentrep_m, y_m),(\instance_n, \latentrep_n, y_n))\}_{m,n}$ where $\instance_n,\instance_m\in{\Selected_t}$ and $y_n,y_m\in\by_t$, we define the \emph{importance-weighted penalty function} as
\begin{align}
    \tilde{p}_{ij} = p_{ij} w_{ij} \quad \text{with} \quad  w_{ij}= \frac{e^{\tempParam(y_i + y_j)}} {\sum\limits_{(m,n) \in D_t}{e^{\tempParam (y_m + y_n)}}}.\label{eq:weighted}
\end{align}
Here the importance weight $\tempParam$ is used to control the aggressiveness of the weighting strategy. 

Combining the kernel learning objective---negative log likelihood and the collision penalty for GP, we define the \emph{pair loss} function $L_{\regWeight,\regParam,\tempParam, g}$ as
\begin{align}\label{eq:pairloss}
L_{\regWeight,\regParam,\tempParam, g}(\algParam_t, D_t)
=
-\log(P(\mathbf{y}_t|\Selected_t,\algParam_t)) + \frac{\regWeight}{||D_t||^2}\sum\limits_{i\in{D_t}, j\in{D_t}}{\tilde{p}_{ij}}
\end{align}
%
where $-\log(P(\mathbf{y}_t|\Selected_t,\algParam_t)) = -\frac{1}{2}\mathbf{y}_t^\top(\GramMat_t+\sigma^2I)^{-1}\by_t -\frac{1}{2}|(\GramMat_t+\sigma^2I)| -\frac{t}{2}\log(2\pi)$ is the learning objective for the GP \citep{10.5555/1162254}. $\regWeight$ denotes the regularization weight; as we demonstrate in \Secref{sec:experiments}, \rebuttal{we initialize the regularization weight} $\regWeight$ to keep the penalty at the same order of magnitude as the negative log likelihood. Another option for optimizing \eqref{eq:pairloss} is to minimize the regression loss and the collision penalty alternatively. We observe in our empirical study that both training processes could lead to reasonable convergence behavior of the \algname training loss.


\section{\rebuttal{Theoretical Insight}} \label{sec:analysis}



This subsection discusses the theoretical insight underlying the collision-free regularizer, by inspecting the effect of regularization on the regret bound of \algname where the constantly trained neural network feeds a dynamic embedding to the Gaussian process. 

While the key idea for bounding the regret of \UCB-based GP bandit optimization algorithms follows the analysis of \citet{10.5555/3104322.3104451}, two unique challenges are posed in the analysis of \algname. Firstly, unlike previous work in \citet{10.5555/3104322.3104451}, the neural network is constantly retrained along with the new observations. Thus, the input space for the downstream Gaussian process could be highly variant. 

\rebuttal{For the discussion below, we consider a stationary and monotonic kernel, and assume that retraining the neural network $g$ does not decrease the distance between data points in the latent space. It is worth noting that, although not strictly enforced, such monotonicity behavior is naturally encouraged by our proposed regularization, which only penalizes the pair of too-close data in the latent space.}
%
%
%
\rebuttal{Under the above assumption, the internal complexity of neural network training still makes it challenging to bound the regret w.r.t the dynamics of the neural network. Thus, we investigate the dynamics of the mutual information term in the regret bound, and justify the proposed collision-free regularizer by showing that penalizing the collisions tends to reduces the upper bound on the regret.}


We first consider a discrete decision set and then leverage the desired Lipschitz continuity on the regularized space to extend our results to the continuous setting \rebuttal{(cf full proofs in \appref{app:proofs})}. 

\begin{proposition}\label{prp:discrete_regret}
    Let $\mathcal{Z}$ be a finite discrete set. 
    Let $\delta \in (0,1)$, and define $\beta_t =2\log(|\mathcal{Z}|t^2/6\delta)$. Suppose that the objective function $h:\mathcal{Z}\times{\mathcal{\algParam}} \xrightarrow[]{}\mathcal{R}$ defined on $\mathcal{Z}$ and parameterized by $\theta$ is a sample from GP. 
    \rebuttal{Furthermore, consider a stationary and monotonic kernel, and assume that retraining the neural network $g$ does not decrease the distance between data points in the latent space.} 
    Running GP-UCB with $\beta_t$ for a sample $h$ of a GP with mean function zero and stationary covariance function $k(\instance,\instance')$, we obtain a regret bound 
    of $\mathcal{O}^*(\sqrt{\log(|\mathcal{Z}|)T(\maxInfo_{T}-\mutualinfo{h(\latentrep_T, \algParam_{h,0}); \indexRVar_T}})$ with high probability. 
    
    More specifically, with $C_1 = 8/\log(1+\sigma^{-2})$, we have 
    $$\Pr{R_T\leq \sqrt{C_1T\beta_T(\maxInfo_{T}-\mutualinfo{h(\latentrep_T, \algParam_{h,0}); \indexRVar_T}})}\geq 1-\delta.$$
    Here $\maxInfo_T$ is the maximum information gain after T iterations, and $\indexRVar_T$ as the identification of the collided data points on $\mathcal{Z}$. $\maxInfo_T$ is defined as $\maxInfo_T \coloneqq \max\limits_{\Selected\subset \mathcal{Z}, |\Selected|=T}\mutualinfo{y_\Selected, \indexRVar_T; h(\Selected, \algParam_T)}$.
\end{proposition}

The collision regularization reduced the maximum mutual information by a specific term dependent on the distribution of the noise caused by the collision of data points. The distribution is dynamic and determined by the complex learning process of the neural network. In the following, we show that the mutual information is bounded within a given interval:

Assume $\indexRVar_t$ is a random variable that identify $\latentrep_t \in \mathcal{Z}$, 
$y_t \in \mathcal{Y}$ from its collided points, and the variance of the collision is $\sigma_{\col}^2$, then we have
$$ 0\leq \mutualinfo{h(\latentrep_T, \algParam_{h,0}); \indexRVar_T} \leq 1/2\log|2\pi e \sigma^2_{\col}I|$$

This means that if $\indexRVar_t$ is a random variable sampled from a Gaussian distribution defined on $h$, then $\mutualinfo{h(\latentrep_T, \algParam_{h,0}); \indexRVar_T}$ is maximized.

The regularization also constrains the function $h$ to be Lipschitz-continuous with a Lipschitz constant, enabling a slightly narrower regret bound.

\begin{proposition}\label{prp:continuous_regret}
    Let $\mathcal{Z} \subset [0, r]^d$  be compact and convex, $d\in N, r>0$ and $\lambda\geq 0$. 
    Suppose that the objective function $h:\mathcal{Z}\times{\mathcal{\algParam}} \xrightarrow[]{}\mathcal{R}$ defined on $\mathcal{Z}$ and parameterized by $\theta$ is a sample from GP and is Lipschitz continuous with Lipschitz constant $\lambda$.
    Let $\delta \in (0,1)$, and define $\beta_t =2\log(\pi^2t^2/6\delta)+2d\log(\lambda rdt^2)$. \rebuttal{Furthermore, consider a stationary and monotonic kernel, and assume that retraining the neural network $g$ does not decrease the distance between data points in the latent space.}
    Running GP-UCB with $\beta_t$ for a sample $h$ of a GP with mean function zero and stationary covariance function $k(\instance, \instance')$, we obtain a regret bound 
    of $O^*(\sqrt{dT(\maxInfo_{T}-\mutualinfo{h(\latentrep_T, \algParam_{h,0}); \indexRVar_T}})$ with high probability. 
    
    More specifically, with $C_1 = 8/\log(1+\sigma^{-2})$, we have 
    $$\Pr{R_T\leq \sqrt{C_1T\beta_T(\maxInfo_{T}-\mutualinfo{h(\latentrep_T, \algParam_{h,0}); \indexRVar_T})}+2}\geq 1-\delta.$$
    Here $\maxInfo_T$ is the maximum information gain after T iterations, and $\indexRVar_T$ as the identification of the collided data points on $\mathcal{Z}$. $\maxInfo_T$ is defined as $\maxInfo_T \coloneqq \max\limits_{\Selected\subset \mathcal{Z}, |\Selected|=T}\mutualinfo{y_\Selected, \indexRVar_T; h(\Selected, \algParam_T)}$.
\end{proposition}

Monotonically increasing pairwise distance also increases the uncertainty throughout the domain if stationary kernels like RBF kernel are applied. Therefore it would not be helpful to merely increase distances as a general practice to improve the smoothness unless the collision is addressed in such unfolding of the space. 

\section{Experiments} \label{sec:experiments}

\begin{figure*}[!t]
     \centering
     \begin{subfigure}[b]{0.25\textwidth}
         \centering
         \includegraphics[trim={0pt 10pt 0pt 0pt},
         width=\textwidth]{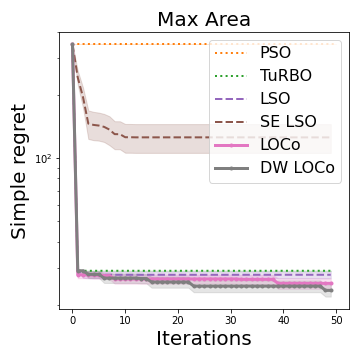}
         \caption{Max Area}
         \label{fig:max_area}
     \end{subfigure}
     \quad
      \begin{subfigure}[b]{0.25\textwidth}
         \centering
         \includegraphics[trim={0pt 10pt 0pt 0pt},
         width=\textwidth]{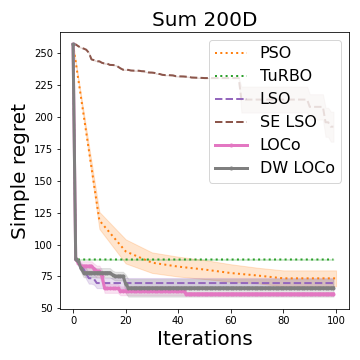}
         \caption{Sum 200D}
         \label{fig:hdbo}
     \end{subfigure}
     \quad
     \begin{subfigure}[b]{0.25\textwidth}
         \centering
         \includegraphics[trim={0pt 10pt 0pt 0pt},
         width=\textwidth]{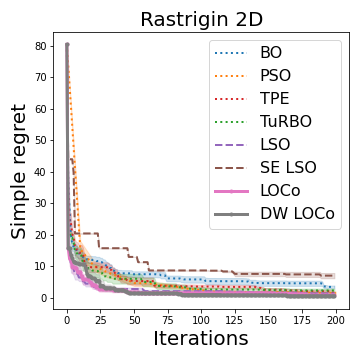}
         \caption{Rastrigin-2D}
         \label{fig:rastrigin}
     \end{subfigure}
     \\
     \begin{subfigure}[b]{0.25\textwidth}
         \centering
         \includegraphics[trim={0pt 10pt 0pt 0pt},
         width=\textwidth]{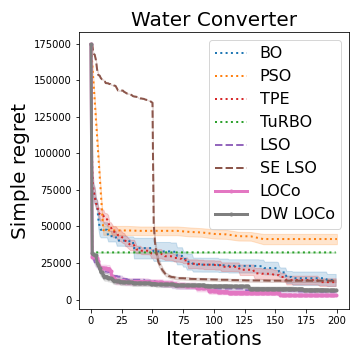}
         \caption{Water Converter}
         \label{fig:water_converter}
     \end{subfigure}
     \quad
     \begin{subfigure}[b]{0.25\textwidth}
         \centering
         \includegraphics[trim={0pt 10pt 0pt 0pt},
         width=\textwidth]{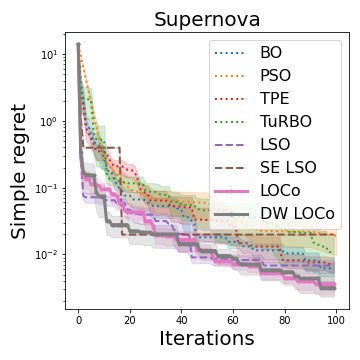}
         \caption{Supernova}
         \label{fig:supernova}
     \end{subfigure}
     \quad
     \begin{subfigure}[b]{0.25\textwidth}
         \centering
         \includegraphics[trim={0pt 10pt 0pt 0pt},
         width=\textwidth]{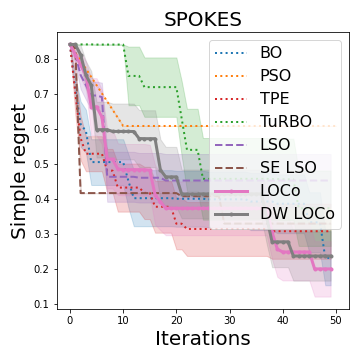}
         \caption{SPOKES}
         \label{fig:spokes}
     \end{subfigure}
    \caption{Experiment results on six pre-collected datasets. Each experiment is repeated at least eight times. The colored area around the mean curve denotes the denotes the standard error.
    Note that the BO and TPE implementations we employed did not terminate in reasonable time for \selffigref{fig:max_area} and \selffigref{fig:hdbo}.}
    \label{fig:six_graphs}
    \vspace{-4mm}
\end{figure*}


In this section, we empirically evaluate our algorithm on several synthetic and real-world benchmark blackbox function optimization tasks. All experiments are conducted on Google Cloud GPU instance (4 vCPUs, 15 GB memory, Tesla T4 GPU) and Google CoLab high-RAM GPU instance.

\subsection{Experimental Setup}
We consider \rebuttal{five} baselines in our experiments. Three popular optimization algorithms---particle swarm optimization (PSO) \citep{pyswarmsJOSS2018}, Tree-structured Parzen Estimator Approach (TPE) \citep{NIPS2011_86e8f7ab}, \rebuttal{a BoTorch \citep{balandat2020botorch} implementation of Trust Region Bayesian Optimization (TuRBO) \citep{eriksson2019scalable}}, and standard Bayesian optimization (BO) \citep{gpopt} which uses Gaussian processes as the statistical model---are tuned in each task. 
Another baseline we consider is the sample-efficient LSO (SE LSO) algorithm, which is implemented based on the algorithm proposed by \cite{Tripp2020SampleEfficientOI}. We also compare the non-regularized latent space optimization (LSO),
\algname with uniform weights (i.e. $\tempParam=0$, referred to as \algname), and the dynamically-weighted \algname (i.e. with $\tempParam>0$, referred to as DW \algname) proposed in this paper.

One crucial problem in practice is tuning the hyper-parameters. The hyper-parameters for GP are tuned for periodically retraining in the optimization process by minimizing the loss function on a validation set. For all our tasks, we choose a simple neural network architecture due to the reasoning in \secref{sec:collision}, as well as due to limited and expensive access to labeled data under the BO setting.
The coefficient $\regWeight$ is, in general, selected to guarantee a similar order for the collision penalty to GP loss. The $\regParam$ should be tolerant of the additive noise in the evaluation. In practice, we choose the simple setting $ \regParam = 1$ and find it perform well. We also include a study of the parameter choice in the appendix. $\tempParam$ controls the aggressiveness of the importance weight. While $\tempParam$ should not be too close to zero (equivalent to uniform weight) , an extremely high value could make the regularization overly biased. Such a severe bias could allow a heavily collided representation in most of the latent space and degrade regularization effectiveness. The value choice is similar to the inverse of the temperature parameter of softmax in deep learning \citep{Hinton2015DistillingTK}. Here we use $\tempParam=1$ for simplicity and find it robust to different tasks. All experiments are conducted on the pre-collected datasets. We defer the detailed experimental setup to \appref{app:expsetup}.

\subsection{Datasets and Results}\label{sec:results}

We now evaluate \algname on three synthetic datasets and three real-world datasets. We demonstrated the improvement in \algname that is enabled by the explicit collision mitigation in the lower-dimensional latent space in terms of average simple regret. 

\paragraph{Max Area-4096D}

The dSprites dataset \citep{dsprites17} consists of images of size 64 $\times$ 64 containing 2d Shapes with different scales, rotations, and positions. Each pixel value of the images are binary, hence $\instance\in \{0,1\}^{64\times64}$. The goal is to generate a shape $\instance$ with maximum area, which is equivalent to finding  $\argmax\limits_{\instance}\sum_{i}^{64\times64} x_i$ where $i$ corresponds to the pixel index and $x_i$ is the $i$th entry of $\instance$. The neural network is pretrained on 50 data points. To meet the limitation of memory on our computing instance, we uniformly sample 10000 points from the original dataset and approximately maintain the original distribution of the objective value. The DW \algname and \algname outperform or match the baseline methods on this dataset.

\paragraph{Sum-200D}
We create a synthetic dataset Sum-200D of 200 dimensions. Each dimension is independently sampled from a standard normal distribution to maximize the uncertainty on that dimensions and examine the algorithm's capability to solve the medium-dimensional problem. We want to maximize the label $f(\instance) = \sum^{200}_{i=1}{e^{x_i}}$ which bears an additive structure and of non-linearity. The neural network is pretrained on 100 data points. As illustrated by \figref{fig:hdbo}, DW \algname and \algname could significantly outperform baselines that do not specifically leverage the additive structures of the problem. 

\paragraph{Rastrigin-2D}
The Rastrigin function is a non-convex function used as a performance test problem for optimization algorithms. It was first proposed by \citet{10018403158} and used as a popular benchmark dataset for evaluating Gaussian process regression algorithms \citep{cully2018limbo}. Concretely, the 2D Rastrigin function is
$f(\instance) = 10{d} + \sum^{d}_{i=1}{x_i^2 - 10\cos(2\pi{x_i})},\ d=2$.
For convenience of comparison, we take the $-f(\instance)$ as the objective value to make the optimization tasks a maximization task.




{\paragraph{Water Converter Configuration-16D}
This UCI dataset we use consists of positions and absorbed power outputs of wave energy converters (WECs) from the southern coast of Sydney. The applied converter model is a fully submerged three-tether converter called CETO. 16 WECs locations are placed and optimized in a size-constrained environment.}

\paragraph{Supernova-3D}

Our first real-world task is to perform maximum likelihood inference on three cosmological parameters, the Hubble constant $H_0\in (60,80)$, the dark matter fraction $\Omega_M\in (0,1)$, and the dark energy fraction $\Omega_A\in (0,1)$. The likelihood is given by the Robertson-Walker metric, which requires a one-dimensional numerical integration for each point in the dataset from \cite{Davis_2007}. The neural network is pretrained on 100 data points. As illustrated by \figref{fig:supernova}, both \algname and DW \algname demonstrate its consistent robustness. Among them, DW \algname slightly outperforms \algname in the early stage.

\paragraph{SPOKES-14D}
Careful accounting of all the requirements and features of these experiments becomes increasingly necessary to achieve the goals of a given cosmic survey.
SPOKES (SPectrOscopic KEn Simulation) is an end-to-end framework that can simulate all the operations and critical decisions of a cosmic survey \citep{NORD20161}.
In this work, we use SPOKES to generate galaxies within a specified window of distances from Earth.
We then minimize the Hausdorff distance between the desired redshift distribution and the simulation of specific cosmological surveys generated by SPOKES. In our experiments, the neural network is pretrained with 400 data points. As illustrated by \figref{fig:spokes}, the simple regret of DW \algname drops slower yet eventually outperforms or matches other baselines' performances. 

In general, our experimental results consistently demonstrate the robustness of our methods against collisions in the learned latent space. Our method outperforms or matches the performance of the best baselines in all scenarios. When compared to the sample-efficient LSO, DW \algname performs better in most cases and shows a steady capability to reach the optimum by explicitly mitigating the collision in the latent space. \rebuttal{Due to the dynamics of representation learning process, it is difficult to claim that the performance improvement brought by dynamic weighting 
is universal. This aligned with the observation in the experiments that DW \algname brings observable improvement in the regret curve at a certain stage for an optimization task and achieve an ultimate performance that at least matches \algname.}
In contrast, the sample-efficient LSO might fail due to the collision problem.


\section{Conclusion}\label{sec:conclusion}
We have proposed a novel regularization scheme for latent-space-based Bayesian optimization. Our algorithm addresses the collision problem induced by dimensionality reduction and improves the performance for latent space-based optimization algorithms. We show that the regularization effectively mitigates the collision problem in the learned latent spaces and, therefore, can boost the performance of the Bayesian optimization in the latent space. We demonstrate solid empirical results for \algname on several synthetic and real-world datasets. Furthermore, we demonstrate that \algname can deal with high-dimensional input that could be highly valuable for real-world experiment design tasks such as cosmological survey scheduling. 

\section*{Acknowledgment}

The authors thank James Bowden, Jialin Song, Thomas Desautels, and Yisong Yue for the helpful discussions. The project was supported in part by NSF grant \#2037026 and a JTFI AI + Science Grant provided by the Center for Data and Computing (CDAC) at the University of Chicago. Any opinions, findings, and conclusions or recommendations expressed in this material are those of the authors and do not necessarily reflect the views of any funding agencies.

\bibliography{references}

\begin{thebibliography}{46}
\providecommand{\natexlab}[1]{#1}
\providecommand{\url}[1]{\texttt{#1}}
\expandafter\ifx\csname urlstyle\endcsname\relax
  \providecommand{\doi}[1]{doi: #1}\else
  \providecommand{\doi}{doi: \begingroup \urlstyle{rm}\Url}\fi

\bibitem[Balandat et~al.(2020)Balandat, Karrer, Jiang, Daulton, Letham, Wilson,
  and Bakshy]{balandat2020botorch}
Maximilian Balandat, Brian Karrer, Daniel~R. Jiang, Samuel Daulton, Benjamin
  Letham, Andrew~Gordon Wilson, and Eytan Bakshy.
\newblock {BoTorch: A Framework for Efficient Monte-Carlo Bayesian
  Optimization}.
\newblock In \emph{Advances in Neural Information Processing Systems 33}, 2020.
\newblock URL \url{http://arxiv.org/abs/1910.06403}.

\bibitem[Bengio et~al.(2005)Bengio, Delalleau, and Le~Roux]{bengio2005curse}
Yoshua Bengio, Olivier Delalleau, and Nicolas Le~Roux.
\newblock The curse of dimensionality for local kernel machines.
\newblock \emph{Techn. Rep}, 1258:\penalty0 12, 2005.

\bibitem[Bergstra et~al.(2011)Bergstra, Bardenet, Bengio, and
  K\'{e}gl]{NIPS2011_86e8f7ab}
James Bergstra, R\'{e}mi Bardenet, Yoshua Bengio, and Bal\'{a}zs K\'{e}gl.
\newblock Algorithms for hyper-parameter optimization.
\newblock In J.~Shawe-Taylor, R.~Zemel, P.~Bartlett, F.~Pereira, and K.~Q.
  Weinberger, editors, \emph{Advances in Neural Information Processing
  Systems}, volume~24, pages 2546--2554. Curran Associates, Inc., 2011.

\bibitem[Berkenkamp et~al.(2016)Berkenkamp, Schoellig, and
  Krause]{berkenkamp2016safe}
Felix Berkenkamp, Angela~P Schoellig, and Andreas Krause.
\newblock Safe controller optimization for quadrotors with gaussian processes.
\newblock 2016.

\bibitem[Binois et~al.(2015)Binois, Ginsbourger, and
  Roustant]{binois2015warped}
Micka{\"e}l Binois, David Ginsbourger, and Olivier Roustant.
\newblock A warped kernel improving robustness in bayesian optimization via
  random embeddings.
\newblock In \emph{International Conference on Learning and Intelligent
  Optimization}, pages 281--286. Springer, 2015.

\bibitem[Bogunovic et~al.(2018)Bogunovic, Scarlett, Jegelka, and
  Cevher]{bogunovic2018adversarially}
Ilija Bogunovic, Jonathan Scarlett, Stefanie Jegelka, and Volkan Cevher.
\newblock Adversarially robust optimization with gaussian processes.
\newblock In \emph{NeurIPS}, 2018.

\bibitem[Cully et~al.(2018)Cully, Chatzilygeroudis, Allocati, and
  Mouret]{cully2018limbo}
A.~Cully, K.~Chatzilygeroudis, F.~Allocati, and J.-B. Mouret.
\newblock {Limbo: A Flexible High-performance Library for Gaussian Processes
  modeling and Data-Efficient Optimization}.
\newblock \emph{{The Journal of Open Source Software}}, 3\penalty0
  (26):\penalty0 545, 2018.
\newblock \doi{10.21105/joss.00545}.

\bibitem[Davis et~al.(2007)Davis, Mortsell, Sollerman, Becker, Blondin,
  Challis, Clocchiatti, Filippenko, Foley, Garnavich, Jha, Krisciunas,
  Kirshner, Leibundgut, Li, Matheson, Miknaitis, Pignata, Rest, Riess, Schmidt,
  Smith, Spyromilio, Stubbs, Suntzeff, Tonry, Wood-Vasey, and
  Zenteno]{Davis_2007}
T.~M. Davis, E.~Mortsell, J.~Sollerman, A.~C. Becker, S.~Blondin, P.~Challis,
  A.~Clocchiatti, A.~V. Filippenko, R.~J. Foley, P.~M. Garnavich, S.~Jha,
  K.~Krisciunas, R.~P. Kirshner, B.~Leibundgut, W.~Li, T.~Matheson,
  G.~Miknaitis, G.~Pignata, A.~Rest, A.~G. Riess, B.~P. Schmidt, R.~C. Smith,
  J.~Spyromilio, C.~W. Stubbs, N.~B. Suntzeff, J.~L. Tonry, W.~M. Wood-Vasey,
  and A.~Zenteno.
\newblock Scrutinizing exotic cosmological models using {ESSENCE} supernova
  data combined with other cosmological probes.
\newblock \emph{The Astrophysical Journal}, 666\penalty0 (2):\penalty0
  716--725, sep 2007.
\newblock \doi{10.1086/519988}.

\bibitem[Ding et~al.(2020)Ding, Xu, Xu, Parmar, Yang, Welling, and
  Tu]{ding2020guided}
Zheng Ding, Yifan Xu, Weijian Xu, Gaurav Parmar, Yang Yang, Max Welling, and
  Zhuowen Tu.
\newblock Guided variational autoencoder for disentanglement learning.
\newblock In \emph{Proceedings of the IEEE/CVF Conference on Computer Vision
  and Pattern Recognition}, pages 7920--7929, 2020.

\bibitem[Djolonga et~al.(2013)Djolonga, Krause, and
  Cevher]{Djolonga2013HighDimensionalGP}
Josip Djolonga, Andreas Krause, and V.~Cevher.
\newblock High-dimensional gaussian process bandits.
\newblock In \emph{Neural Information Processing Systems}, 2013.

\bibitem[Eriksson et~al.(2019)Eriksson, Pearce, Gardner, Turner, and
  Poloczek]{eriksson2019scalable}
David Eriksson, Michael Pearce, Jacob Gardner, Ryan~D Turner, and Matthias
  Poloczek.
\newblock Scalable global optimization via local bayesian optimization.
\newblock \emph{Advances in Neural Information Processing Systems},
  32:\penalty0 5496--5507, 2019.

\bibitem[Ferreira et~al.(2020)Ferreira, Camacho, and
  Teixeira]{ferreira2020using}
Mafalda~Falc{\~a}o Ferreira, Rui Camacho, and Lu{\'\i}s~F Teixeira.
\newblock Using autoencoders as a weight initialization method on deep neural
  networks for disease detection.
\newblock \emph{BMC Medical Informatics and Decision Making}, 20\penalty0
  (5):\penalty0 1--18, 2020.

\bibitem[Galuzzi et~al.(2019)Galuzzi, Giordani, Candelieri, Perego, and
  Archetti]{galuzzi2019bayesian}
Bruno~Giovanni Galuzzi, Ilaria Giordani, Antonio Candelieri, Riccardo Perego,
  and Francesco Archetti.
\newblock Bayesian optimization for recommender system.
\newblock In \emph{World Congress on Global Optimization}, pages 751--760.
  Springer, 2019.

\bibitem[Gardner et~al.(2018)Gardner, Pleiss, Bindel, Weinberger, and
  Wilson]{gardner2018gpytorch}
Jacob~R Gardner, Geoff Pleiss, David Bindel, Kilian~Q Weinberger, and
  Andrew~Gordon Wilson.
\newblock Gpytorch: Blackbox matrix-matrix gaussian process inference with gpu
  acceleration.
\newblock In \emph{Advances in Neural Information Processing Systems}, 2018.

\bibitem[G{\'o}mez-Bombarelli et~al.(2018)G{\'o}mez-Bombarelli, Wei, Duvenaud,
  Hern{\'a}ndez-Lobato, S{\'a}nchez-Lengeling, Sheberla, Aguilera-Iparraguirre,
  Hirzel, Adams, and Aspuru-Guzik]{gomez2018automatic}
Rafael G{\'o}mez-Bombarelli, Jennifer~N Wei, David Duvenaud, Jos{\'e}~Miguel
  Hern{\'a}ndez-Lobato, Benjam{\'\i}n S{\'a}nchez-Lengeling, Dennis Sheberla,
  Jorge Aguilera-Iparraguirre, Timothy~D Hirzel, Ryan~P Adams, and Al{\'a}n
  Aspuru-Guzik.
\newblock Automatic chemical design using a data-driven continuous
  representation of molecules.
\newblock \emph{ACS central science}, 4\penalty0 (2):\penalty0 268--276, 2018.

\bibitem[Griffiths and Hern{\'a}ndez-Lobato(2020)]{griffiths2020constrained}
Ryan-Rhys Griffiths and Jos{\'e}~Miguel Hern{\'a}ndez-Lobato.
\newblock Constrained bayesian optimization for automatic chemical design using
  variational autoencoders.
\newblock \emph{Chemical science}, 11\penalty0 (2):\penalty0 577--586, 2020.

\bibitem[Grosnit et~al.(2021)Grosnit, Tutunov, Maraval, Griffiths,
  Cowen-Rivers, Yang, Zhu, Lyu, Chen, Wang, et~al.]{grosnit2021high}
Antoine Grosnit, Rasul Tutunov, Alexandre~Max Maraval, Ryan-Rhys Griffiths,
  Alexander~I Cowen-Rivers, Lin Yang, Lin Zhu, Wenlong Lyu, Zhitang Chen, Jun
  Wang, et~al.
\newblock High-dimensional bayesian optimisation with variational autoencoders
  and deep metric learning.
\newblock \emph{arXiv preprint arXiv:2106.03609}, 2021.

\bibitem[Hinton et~al.(2015)Hinton, Vinyals, and Dean]{Hinton2015DistillingTK}
Geoffrey~E. Hinton, Oriol Vinyals, and J.~Dean.
\newblock Distilling the knowledge in a neural network.
\newblock \emph{ArXiv}, abs/1503.02531, 2015.

\bibitem[Huang et~al.(2015)Huang, Zhao, Sun, Liu, and
  Chang]{10.5555/2832747.2832748}
Wenbing Huang, Deli Zhao, Fuchun Sun, Huaping Liu, and Edward Chang.
\newblock Scalable gaussian process regression using deep neural networks.
\newblock In \emph{Proceedings of the 24th International Conference on
  Artificial Intelligence}, IJCAI'15, pages 3576 -- 3582. AAAI Press, 2015.
\newblock ISBN 9781577357384.

\bibitem[Kingma and Ba(2014)]{kingma2014adam}
Diederik~P Kingma and Jimmy Ba.
\newblock Adam: A method for stochastic optimization.
\newblock \emph{arXiv preprint arXiv:1412.6980}, 2014.

\bibitem[Letham et~al.(2020)Letham, Calandra, Rai, and Bakshy]{letham2020re}
Ben Letham, Roberto Calandra, Akshara Rai, and Eytan Bakshy.
\newblock Re-examining linear embeddings for high-dimensional bayesian
  optimization.
\newblock \emph{Advances in Neural Information Processing Systems}, 33, 2020.

\bibitem[Lu et~al.(2018)Lu, Gonzalez, Dai, and Lawrence]{pmlr-v80-lu18c}
Xiaoyu Lu, Javier Gonzalez, Zhenwen Dai, and Neil Lawrence.
\newblock Structured variationally auto-encoded optimization.
\newblock volume~80 of \emph{Proceedings of Machine Learning Research}, pages
  3267--3275, Stockholm International Fairs, Stockholm Sweden, 10--15 Jul 2018.
  PMLR.

\bibitem[Mathieu et~al.(2019)Mathieu, Rainforth, Siddharth, and
  Teh]{mathieu2019disentangling}
Emile Mathieu, Tom Rainforth, Nana Siddharth, and Yee~Whye Teh.
\newblock Disentangling disentanglement in variational autoencoders.
\newblock In \emph{International Conference on Machine Learning}, pages
  4402--4412. PMLR, 2019.

\bibitem[Matthey et~al.(2017)Matthey, Higgins, Hassabis, and
  Lerchner]{dsprites17}
Loic Matthey, Irina Higgins, Demis Hassabis, and Alexander Lerchner.
\newblock dsprites: Disentanglement testing sprites dataset.
\newblock https://github.com/deepmind/dsprites-dataset/, 2017.

\bibitem[Miranda(2018)]{pyswarmsJOSS2018}
Lester James~V. Miranda.
\newblock {P}y{S}warms, a research-toolkit for {P}article {S}warm
  {O}ptimization in {P}ython.
\newblock \emph{Journal of Open Source Software}, 3, 2018.
\newblock \doi{10.21105/joss.00433}.

\bibitem[Mutn{\`y} and Krause(2019)]{mutny2019efficient}
Mojm{\'\i}r Mutn{\`y} and Andreas Krause.
\newblock Efficient high dimensional bayesian optimization with additivity and
  quadrature fourier features.
\newblock \emph{NeurIPS}, 2019.

\bibitem[Nayebi et~al.(2019)Nayebi, Munteanu, and
  Poloczek]{nayebi2019framework}
Amin Nayebi, Alexander Munteanu, and Matthias Poloczek.
\newblock A framework for bayesian optimization in embedded subspaces.
\newblock In \emph{International Conference on Machine Learning}, pages
  4752--4761. PMLR, 2019.

\bibitem[Nogueira(2014)]{gpopt}
Fernando Nogueira.
\newblock {Bayesian Optimization}: Open source constrained global optimization
  tool for {Python}, 2014.

\bibitem[Nord et~al.(2016)Nord, Amara, Réfrégier, Gamper, Gamper, Hambrecht,
  Chang, Forero-Romero, Serrano, Cunha, Coles, Nicola, Busha, Bauer, Saunders,
  Jouvel, Kirk, and Wechsler]{NORD20161}
B.~Nord, A.~Amara, A.~Réfrégier, La. Gamper, Lu. Gamper, B.~Hambrecht,
  C.~Chang, J.E. Forero-Romero, S.~Serrano, C.~Cunha, O.~Coles, A.~Nicola,
  M.~Busha, A.~Bauer, W.~Saunders, S.~Jouvel, D.~Kirk, and R.~Wechsler.
\newblock Spokes: An end-to-end simulation facility for spectroscopic
  cosmological surveys.
\newblock \emph{Astronomy and Computing}, 15:\penalty0 1 -- 15, 2016.
\newblock ISSN 2213-1337.

\bibitem[Ober et~al.(2021)Ober, Rasmussen, and van~der Wilk]{ober2021promises}
Sebastian~W Ober, Carl~E Rasmussen, and Mark van~der Wilk.
\newblock The promises and pitfalls of deep kernel learning.
\newblock In \emph{Uncertainty in Artificial Intelligence}, pages 1206--1216.
  PMLR, 2021.

\bibitem[Rahimi et~al.(2007)Rahimi, Recht, et~al.]{rahimi2007random}
Ali Rahimi, Benjamin Recht, et~al.
\newblock Random features for large-scale kernel machines.
\newblock In \emph{Neural Information Processing Systems}, volume~3, page~5.
  Citeseer, 2007.

\bibitem[Rasmussen and Williams(2006)]{rasmussen:williams:2006}
C.~E. Rasmussen and C.~K.~I. Williams.
\newblock \emph{Gaussian Processes for Machine Learning}.
\newblock MIT Press, 2006.

\bibitem[Rasmussen and Williams(2005)]{10.5555/1162254}
Carl~Edward Rasmussen and Christopher K.~I. Williams.
\newblock \emph{Gaussian Processes for Machine Learning (Adaptive Computation
  and Machine Learning)}.
\newblock The MIT Press, 2005.
\newblock ISBN 026218253X.

\bibitem[Rastrigin(1974)]{10018403158}
Leonard~Andreevi{\v{c}} Rastrigin.
\newblock Systems of extremal control.
\newblock \emph{Nauka}, 1974.

\bibitem[Snoek et~al.(2012)Snoek, Larochelle, and Adams]{snoek2012practical}
Jasper Snoek, Hugo Larochelle, and Ryan~P Adams.
\newblock Practical bayesian optimization of machine learning algorithms.
\newblock In \emph{26th Annual Conference on Neural Information Processing
  Systems 2012}, pages 2951--2959, 2012.

\bibitem[Snoek et~al.(2015)Snoek, Rippel, Swersky, Kiros, Satish, Sundaram,
  Patwary, Prabhat, and Adams]{snoek2015scalable}
Jasper Snoek, Oren Rippel, Kevin Swersky, Ryan Kiros, Nadathur Satish,
  Narayanan Sundaram, Mostofa Patwary, Mr~Prabhat, and Ryan Adams.
\newblock Scalable bayesian optimization using deep neural networks.
\newblock In \emph{International conference on machine learning}, pages
  2171--2180. PMLR, 2015.

\bibitem[Srinivas et~al.(2010)Srinivas, Krause, Kakade, and
  Seeger]{10.5555/3104322.3104451}
Niranjan Srinivas, Andreas Krause, Sham Kakade, and Matthias Seeger.
\newblock Gaussian process optimization in the bandit setting: No regret and
  experimental design.
\newblock 2010.

\bibitem[Sui et~al.(2018)Sui, Burdick, Yue, et~al.]{sui2018stagewise}
Yanan Sui, Joel Burdick, Yisong Yue, et~al.
\newblock Stagewise safe bayesian optimization with gaussian processes.
\newblock In \emph{International Conference on Machine Learning}, pages
  4781--4789. PMLR, 2018.

\bibitem[Thompson(1933)]{10.1093/biomet/25.3-4.285}
William~R Thompson.
\newblock On the likelihood that one unknown probability exceeds another in
  view of the evidence of two samples.
\newblock \emph{Biometrika}, 25\penalty0 (3/4):\penalty0 285--294, 1933.

\bibitem[Tripp et~al.(2020)Tripp, Daxberger, and
  Hern{\'a}ndez-Lobato]{Tripp2020SampleEfficientOI}
Austin Tripp, Erik Daxberger, and Jos{\'e}~Miguel Hern{\'a}ndez-Lobato.
\newblock Sample-efficient optimization in the latent space of deep generative
  models via weighted retraining.
\newblock \emph{Advances in Neural Information Processing Systems}, 33, 2020.

\bibitem[Udrescu and Tegmark(2020)]{Udrescu2020AIFA}
Silviu-Marian Udrescu and Max Tegmark.
\newblock {AI Feynman: A physics-inspired method for symbolic regression}.
\newblock \emph{Science Advances}, 6, 2020.

\bibitem[van Amersfoort et~al.(2021)van Amersfoort, Smith, Jesson, Key, and
  Gal]{van2021feature}
Joost van Amersfoort, Lewis Smith, Andrew Jesson, Oscar Key, and Yarin Gal.
\newblock On feature collapse and deep kernel learning for single forward pass
  uncertainty.
\newblock \emph{arXiv preprint arXiv:2102.11409}, 2021.

\bibitem[Wang et~al.(2016)Wang, Hutter, Zoghi, Matheson, and
  de~Feitas]{wang2016bayesian}
Ziyu Wang, Frank Hutter, Masrour Zoghi, David Matheson, and Nando de~Feitas.
\newblock Bayesian optimization in a billion dimensions via random embeddings.
\newblock \emph{Journal of Artificial Intelligence Research}, 55:\penalty0
  361--387, 2016.

\bibitem[Welling and Teh(2011)]{welling2011bayesian}
Max Welling and Yee~W Teh.
\newblock Bayesian learning via stochastic gradient langevin dynamics.
\newblock In \emph{Proceedings of the 28th international conference on machine
  learning (ICML-11)}, pages 681--688. Citeseer, 2011.

\bibitem[Wilson et~al.(2016)Wilson, Hu, Salakhutdinov, and
  Xing]{pmlr-v51-wilson16}
Andrew~Gordon Wilson, Zhiting Hu, Ruslan Salakhutdinov, and Eric~P. Xing.
\newblock Deep kernel learning.
\newblock volume~51 of \emph{Proceedings of Machine Learning Research}, pages
  370--378, Cadiz, Spain, 09--11 May 2016. PMLR.

\bibitem[Yang et~al.(2019)Yang, Wu, and Arnold]{yang2019machine}
Kevin~K Yang, Zachary Wu, and Frances~H Arnold.
\newblock Machine-learning-guided directed evolution for protein engineering.
\newblock \emph{Nature methods}, 16\penalty0 (8):\penalty0 687--694, 2019.

\end{thebibliography}
\onecolumn
\clearpage
\begin{appendix}
\section{Proofs}\label{app:proofs}

In this section, we provide proofs for our main theoretical results (Proposition \ref{prp:discrete_regret} and Proposition \ref{prp:continuous_regret}).

\subsection{\rebuttal{Proof of Proposition \ref{prp:discrete_regret}: 
Regret Bound on Discrete Decision Set}} 




We follow the proof structure in \citet{10.5555/3104322.3104451} and introduce new notations to characterize the learning process of the neural network and the collision in the proof.

Before proving Proposition \ref{prp:discrete_regret}, we first introduce a few useful lemmas.

\begin{lemma}\label{lem: discrete_ucb_1}
Pick $\delta\in(0,1)$ and set $\beta_t=2\log(|\DataSet|\pi_t/\delta)$, where $ \sum_{t\geq1}{\pi_t^{-1}}=1$, $\pi_t > 0$. Then with probability $\geq 1-\delta$, $\forall \instance\in \DataSet, \forall t\geq 1$
$$|h(g(\instance, \algParam_{g,t-1}), \algParam_{t-1}) - \mu_{t-1}| \leq \beta_t^{1/2}\sigma_{t-1}(g(\instance, \algParam_{g,t-1}))$$
Here $\algParam_{g,t-1}$ is the parameter for $g$ at time step $t-1$. $\algParam_{h, t-1}$ is the parameter for $h$ at time step $t-1$.
\end{lemma}

\begin{proof}
Fix $t\geq 1$. $\forall \instance\in \DataSet$, $\{x_1,...,x_{t-1}\}$ is deterministically conditioned on $\by_{t-1}=(y_1, ..., y_{t-1})$, 

and $h(g(\instance, \algParam_{g,t-1}), \algParam_{h, t-1})\sim N(\mu_{t-1}(g(\instance, \algParam_{g,t-1})), \sigma^2_{t-1}(g(\instance, \algParam_{g,t-1})))$.
Then using the subgaussianity of $h$, we have
\begin{align*}
    &\Pr{|h(g(\instance, \algParam_{g,t-1}), \algParam_{t-1}) - \mu_{t-1}| \geq \beta_t^{1/2}\sigma_{t-1}(g(\instance, \algParam_{g,t-1}))} \\
    &\leq e^{-\beta_t/2}
\end{align*}

Applying the union bound, with probability $\geq 1 - |\DataSet|e^{-\beta_t/2}$, $\forall \instance \in \DataSet$

$$|h(g(\instance, \algParam_{g,t-1}), \algParam_{t-1}) - \mu_{t-1}| \leq \beta_t^{1/2}\sigma_{t-1}(g(\instance, \algParam_{g,t-1}))$$

Let $|\DataSet|e^{-\beta_t/2} = \delta/\pi_t$, applying the union bound for $\forall t\in \nats$ the statement holds.
\end{proof}

\begin{lemma}\label{lem: discrete_ucb_2}
\rebuttal{Consider a stationary and monotonic kernel, and assume that retraining the neural network $g$ does not decrease the distance between data points in the latent space. Then $\forall t\geq 1$,
}
$$r_t(\algParam_{t-1}) \leq
2\beta^{1/2}_t\sigma_{t-1}(g(\instance, \algParam_{g,t-1}))
\leq 2\beta^{1/2}_t\sigma_{t-1}(g(\instance, \algParam_{g,T})).$$
Here $\algParam_{g,T}$ is the parameter for $g$ at time step $T$.
\end{lemma}

\begin{proof}
The following holds:
\begin{align*}
r_t(\algParam_{t-1})
& = h(g(x^*, \algParam_{g, t-1}), \algParam_{h,t-1}) - h(g(\instance_t, \algParam_{g, t-1}), \algParam_{h,t-1}) \\
& \leq
\beta^{1/2}_t\sigma_{t-1}(g(x^*, \algParam_{g,t-1})) + \mu_{t-1}(g(x^*, \algParam_{g, t-1}) - h(g(\instance_t, \algParam_{g, t-1}), \algParam_{h,t-1})\\
& \leq
\beta^{1/2}_t\sigma_{t-1}(g(\instance_t, \algParam_{g,t-1})) + \mu_{t-1}(g(\instance_t, \algParam_{g, t-1}) - h(g(\instance_t, \algParam_{g, t-1}), \algParam_{h,t-1}) \\
& \leq
2\beta^{1/2}_t\sigma_{t-1}(g(\instance_t, \algParam_{g,t-1})) \\
& \leq
2\beta^{1/2}_t\sigma_{t-1}(g(\instance_t, \algParam_{g,T}))
\end{align*}
\rebuttal{The last line is because the non-decreasing distance between $g(\instance, \algParam_{g,T})$ and $g(\instance', \algParam_{g,T})$ $\forall \instance, \instance'\in \DataSet$ leads to larger variance $\sigma_t$ when using a stationary and monotonic kernel}.
\end{proof}

\begin{lemma}\label{lem: mutual_information_1}
The information gain for the points selected can be expressed in terms of the predictive variance. If $\bh_T=(h(g(\instance_t,\algParam_{g,t}), \algParam_{h,t}))_{t\in{(1,...,T)}} \in \reals^T$: 

$$\mutualinfo{\by_T;\bh_T}=\frac{1}{2}\sum_{t=1}^{T}\log(1+\sigma^{-2}\sigma^2_{t-1}(g(\instance_t,\algParam_{g,T})))$$
\end{lemma}

\begin{proof}
First, we get $\mutualinfo{\by_T;\bh_T}=H(\by_T)-\frac{1}{2}\log|2\pi e\sigma^2\mathbf{I}|$. Then,
\begin{align*}
H(\by_T) & =H(\by_{T-1})+H(y_{T}|\by_{T-1}) \\
         & =H(\by_{T-1})+\log(2\pi e(\sigma^2_{t-1}(g(\bx_T, \algParam_{g,T}))))
\end{align*}
Since $x_1, \dot, x_T$ are deterministic 
conditioned on $\by_{T-1}$. The result follows by induction.
\end{proof}

\begin{lemma}\label{lem: mutual_information_2}
\rebuttal{The gap of the mutual information between collision-free $g(\instance_t,\algParam_{g,T})$ and unregularized $g(\instance_t,\algParam_{g,0})$ is
\begin{align*}
    \mutualinfo{\by_T; h(\latentrep_T, \algParam_{h, T})} & = \mutualinfo{\by_T; h(\latentrep_T,\algParam_{h, T})| \indexRVar_T } 
    \\
    & = \mutualinfo{\by_T; h(\latentrep_T, \algParam_{h, 0})| \indexRVar_T} 
    \\
    & = \mutualinfo{\by_T, \indexRVar_T; h(\latentrep_T, \algParam_{h, 0})} - \mutualinfo{h(\latentrep_T, \algParam_{h, 0}); \indexRVar_T)}
\end{align*}
Here $ \latentrep_t = g(\instance_t,\algParam_{g,T})$, and $\indexRVar$ is the identification of collided data points.}
\end{lemma}

The result is a simple application of the chain rule of mutual information. $\mutualinfo{\by_T;\bh_T}=\mutualinfo{\by_T; h(\latentrep_T, \algParam_{h, T})}$ corresponds to the information gain under fully regularized and assumed collision-free setting. $\mutualinfo{\by_T, \indexRVar_T; h(\latentrep_T, \algParam_{h, 0})}$ corresponds to information gain under unregularized setting.

\begin{lemma}\label{lem: discrete_regret}
Pick $\delta \in (0,1)$ and let $\beta_t$ be defined as in \lemref{lem: discrete_ucb_1}. $\forall T\geq 1$, the following holds with probability at least $1 - \delta$:
$$ \sum_{t=1}^T r^2_t \leq \beta_TC_1\mutualinfo{\by_T; \bh_T} \leq C_1\beta_T(\maxInfo_T-\mutualinfo{h(\latentrep_T, \algParam_{h,0}); \indexRVar_T}).
$$
Here $C_1 \coloneqq \frac{8}{\log(1+\sigma^{-2})} \geq 8\sigma^2$.
\end{lemma}

\begin{proof}
We first observe that
\begin{align*}
    4\beta_t\sigma^2_{t-1}(g(\instance_t, \algParam_{g, T}))
    & \leq
    4\beta_t\sigma^2(\sigma^{-2}\sigma^2_{t-1}(g(\instance_t, \algParam_{g, T}))) \\
    & \leq
    4\beta_t\sigma^2(\frac{\sigma^{-2}}{\log(1+\sigma^{-2})})\log(1+\sigma^{-2}\sigma^2_{t-1}(g(\instance_t, \algParam_{g, T})))
\end{align*}
Combining the above inequality with \rebuttal{\lemref{lem: discrete_ucb_2}}, \lemref{lem: mutual_information_1} and \lemref{lem: mutual_information_2} completes the proof.
\end{proof}

Now we are ready to prove Proposition \ref{prp:discrete_regret}.
\begin{proof}[Proof of Proposition \ref{prp:discrete_regret}]
\propref{prp:discrete_regret} is a simple consequence of \lemref{lem: mutual_information_2} and \lemref{lem: discrete_regret} and Cauchy-Schwarz inequality.
\end{proof}

\subsection{Proof of Proposition \ref{prp:continuous_regret}: Regret Bound with Lipschitz-Continuous Objective Function}
Proposition \ref{prp:continuous_regret} could be regarded as an application of \citet{10.5555/3104322.3104451}, but requires non-trivial adaptations of their proof technique.
%
In the following, we first modify Lemma 5.7 and Lemma 5.8 in \cite{10.5555/3104322.3104451} since we are assuming the deterministic Lipschitz-continuity for $h$. 
Inspired by the prior work, we use 
$\LatentRepSet_t$ defined as a set of discretization $\LatentRepSet_t \subset \mathcal{Z}$ 
at time $t$ in the analysis. 
We choose a discretization $\LatentRepSet_t$ of size $(\tau_t)^d$. so that $\forall \latentrep \in \mathcal{Z}$,
\begin{equation}\label{eq:discretization}
    ||\latentrep-\latentrep_t||_1 \leq rd/\tau_t
\end{equation}
where $\latentrep_t$ denotes the closest point in $\LatentRepSet_t$ to $\latentrep$.

\begin{lemma}\label{lem: cont_ucb_1}
Pick $\delta \in (0,1)$ and set $\beta=2\log(\pi_t\delta) + 2d\log(Lrdt^2)$, where  $\sum_{t\geq1}\pi^{-1}_t=1, \ \pi_t > 0$. Let $\tau_t=Lrdt^2$. Hence then
$$|h(\latentrep^*, \algParam_{h, t-1})-\mu_{t-1}([\latentrep^*]_t)| \leq \beta^{1/2}_t \sigma_{t-1}([\latentrep^*]_t) + 1/t^2\quad \forall t\geq 1$$
holds with probability $\geq 1-\delta$.
\end{lemma}

\begin{proof}
Using the Lipschitz-continuity and \eqref{eq:discretization}, we have that
$$\forall \latentrep\in \mathcal{Z}, |h(\latentrep, \algParam_{h, t-1}) - h(\latentrep_t,, \algParam_{h, t-1})| \leq Lrd/\tau_t$$
By choosing $\tau_t=Lrdt^2$, we have $|\LatentRepSet_t| = (Lrdt^2)^d$ and
$$\forall \latentrep\in \mathcal{Z}, |h(\latentrep, \algParam_{h, t-1}) - h(\latentrep_t, \algParam_{h, t-1})| \leq 1/t^2$$
Then using \lemref{lem: discrete_ucb_1}, we finish the proof. 
\end{proof}

Based on \lemref{lem: discrete_ucb_2} and \lemref{lem: cont_ucb_1}, we could directly obtain the following result.

\begin{lemma} \label{lem: cont_ucb_2}
Pick $\delta \in (0,1)$ and set $\beta=2\log(2\pi_t\delta) + 2d\log(Lrdt^2)$, where  $\sum_{t\geq1}\pi^{-1}_t=1, \ \pi_t > 0$. Then with probability $\geq 1-\delta$, for all $t \in N$, the regret is bounded as follows:
$$r_t\leq 2\beta_t^{1/2}\sigma_{t-1}(\latentrep_t) + 1/t^2$$
\end{lemma}

\begin{proof}
Using the union bound of $\delta/2$ in both \lemref{lem: discrete_ucb_2} and \lemref{lem: cont_ucb_1}, we have that with probability $1-\delta$:
\begin{align*}
    r_t&=h(\latentrep^*) - h(\latentrep_t) \\
        &\leq \beta_t^{1/2}\sigma_{t-1}(\latentrep_t) + 1/t^2 + \mu_{t-1}(\latentrep_t) - h(\latentrep_t) \\
        &\leq 2\beta_t^{1/2}\sigma_{t-1}(\latentrep_t) + 1/t^2
\end{align*}
which complete the proof.
\end{proof}

Now we are ready to prove Proposition \ref{prp:continuous_regret}.
\begin{proof}[Proof of Proposition \ref{prp:continuous_regret}]
Using \lemref{lem: cont_ucb_2}, we have that with probability $\geq 1-\delta$:
$$\sum_{t=1}^T4\beta_t\sigma^2_{t-1}(\instance_t)\leq C_1\beta_T(\maxInfo_T-\mutualinfo{h(\latentrep_T, \algParam_{h,0}); \indexRVar_T}) \quad \forall T \geq 1$$
By Cauchy-Schwarz:
$$\sum_{t=1}^T 2\beta_t^{1/2}\sigma_{t-1}(\instance_t)\leq \sqrt{C_1\beta_T(\maxInfo_T-\mutualinfo{h(\latentrep_T, \algParam_{h,0}); \indexRVar_T})} \quad \forall T \geq 1$$
Finally, substitute $\pi_t$ with $\pi^2t^2/6$ (since $\sum{1/t^2}=\pi^2/6$). \propref{prp:continuous_regret} follows.
\end{proof}

\section{Demonstration of the Collision Effect}

\subsection{Visualization of the Collision Effect in the Latent Space}\label{app:visual_collision}



In this section, we demonstrate the collision effect in the latent space.

\begin{figure}[!h]
     \centering
     \begin{subfigure}[b]{0.49\textwidth}
         \centering
         \includegraphics[width=\textwidth]{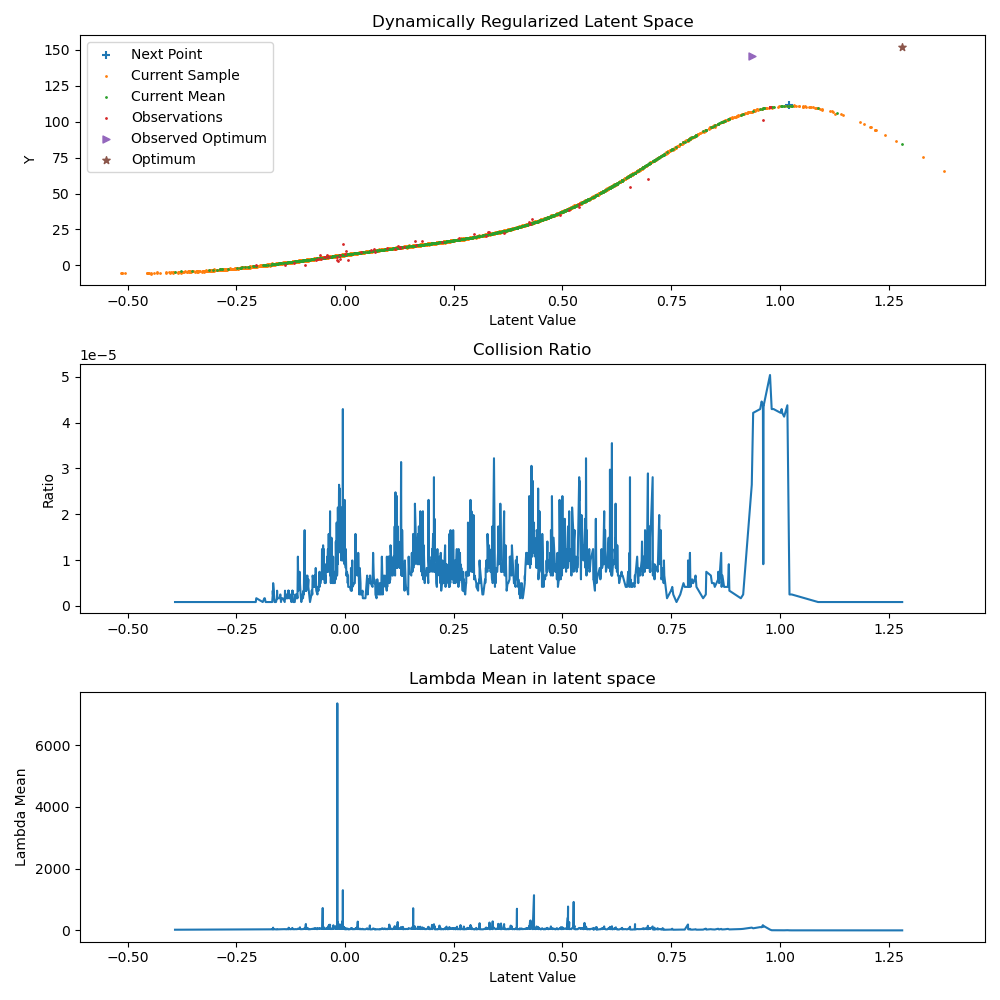}
         \caption{1-D regularized latent space}
         \label{fig:reg_space}
     \end{subfigure}
     \hfill
     \begin{subfigure}[b]{0.49\textwidth}
         \centering
         \includegraphics[width=\textwidth]{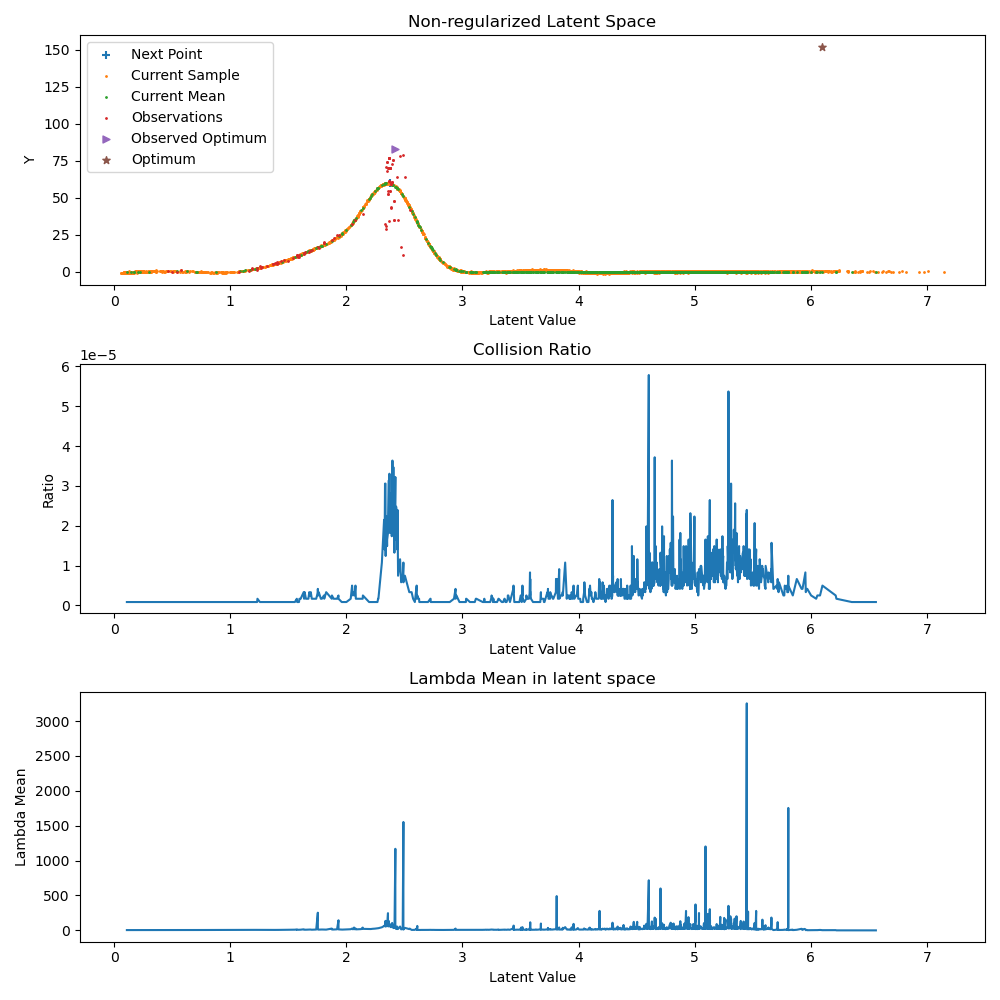}
         \caption{1-D non-regularized latent space}
         \label{fig:non-reg-space}
     \end{subfigure}
     \caption{Illustrate the 1-D latent space of Feynman III.9.52 dataset. The second row shows the ratio that the penalty define as \eqref{eq:penalty} is non-zero. The third row shows point-wise estimation of $\regParam$. \ref{fig:reg_space} shows a regularized latent space with a few observable collisions. \ref{fig:non-reg-space} shows a non-regularized latent space with bumps of collisions especially around the maxima among the observed data points.  Besides, having fewer collisions in the latent space contribute to the optimization through improving the learned Gaussian process. We observe in this comparison that the next point selected by the acquisition function of the regularized version is approaching the global optima, while the next point in the non-regularized version is trying to solve the uncertainty brought by the severe collision near the currently observed maxima.}\label{fig:feynman:collision}
\end{figure}

In \figref{fig:feynman:collision}, we use Feynman dataset which consists of the symbolic regression tasks in physics \citep{Udrescu2020AIFA} and the equation III.9.52 we choose to test is $\regWeight_{\tempParam}=\frac{p_d E_f t}{h/2\pi} \frac{sin((\omega-\omega_0) t/2)^2}{((\omega-\omega_0) t/2)^2}$. We train the same neural network on Feynman dataset with 101 data points which demonstrate the latent space after two retrains with the retrain interval set to be 50 data points. The regularized one employs DW \algname, with the regularization parameter $\rho = 1e^5$, penalty parameter $\lambda=1e^{-2}$, retrain  interval $\tilde{T}$, weighting parameter $\maxInfo=1e^{-2}$ and the base kernel set to be square exponential kernel. The non-regularized one employs LSO.

\begin{figure}[!h]
    \begin{center}
    \includegraphics[width=\textwidth]{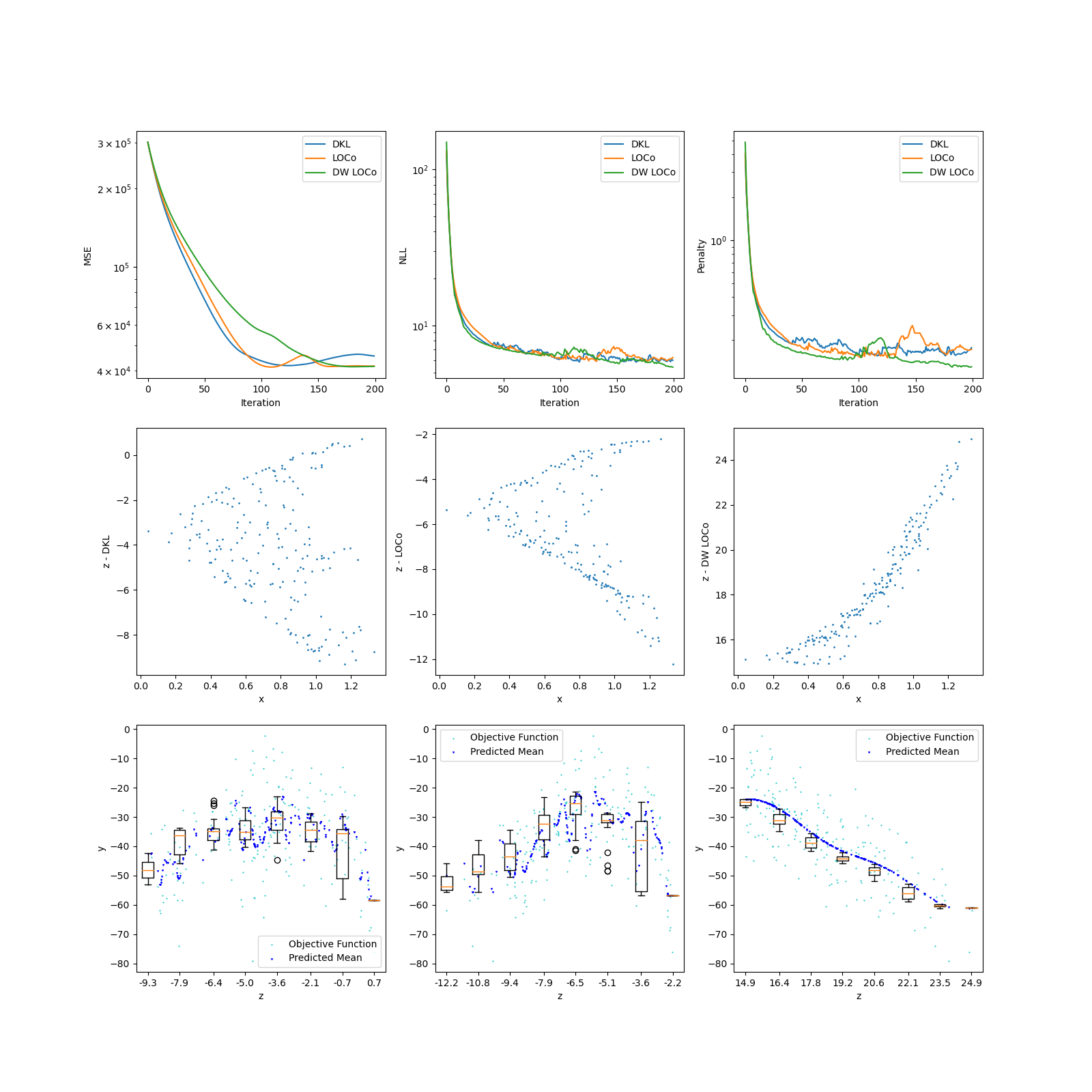}
    \vspace{-10mm}
    \caption{Rastrigin 2D training results. The columns of the bottom two rows corresponds to DKL, \algname and DW \algname. The second row demonstrates the relation between the norm of the 2D input and the corresponding position in the latent space. The third row demonstrates the latent space.}\label{fig:ras2d_training_study_full}
    \end{center}
\end{figure}

\Figref{fig:ras2d_training_study_full} shows the collision effect on the Rastrigin 2D dataset. As illustrated in \figref{fig:ras2d_training_study_full}, even after being sufficiently trained after 200 iterations, there are still collisions in the latent space for DKL. Applying the (dynamically weighted) collision penalty reduces the heterogeneous noise in the latent space, while the MSE and NLL are not significantly different. The observation indicates that improvement of collision-mitigation could be insignificant in regression task as it recovers the space as well as the non-regularized learning.


\subsection{The Collision Effect on Proper Neural Networks}\label{app:collision_vs_nn}

\rebuttal{In this section, we provide empirical results supporting the claim in \secref{sec:collision} that increasing the network complexity often does \textit{not} help to reduce the collision in the latent space.} The results are summarized in \figref{fig:properNNs}.

\newcommand\scaleVar{.3}
\newcommand\widthVar{0.8}
\begin{figure*}[!h]
    \centering
    \begin{subfigure}[b]{\widthVar\textwidth}
        \centering
        \hspace{-4mm}
        \includegraphics[scale=\scaleVar]{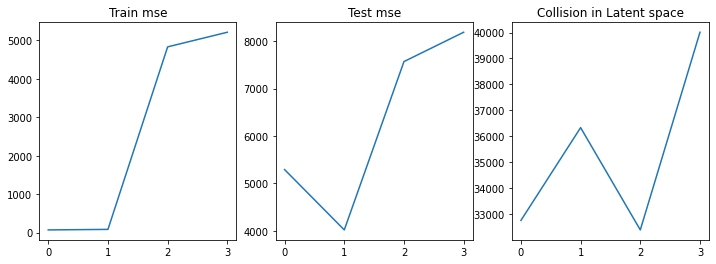}
        \caption{Max Area}
    \end{subfigure}
    \\
    \begin{subfigure}[b]{\widthVar\textwidth}
        \centering
        \includegraphics[scale=\scaleVar,trim={30pt 0pt 0pt 0pt}]{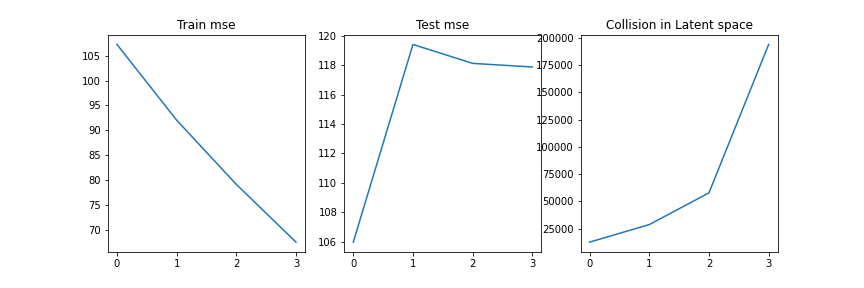}
        \caption{Rastrigin-2D}
    \end{subfigure}
    \\
    \begin{subfigure}[b]{\widthVar\textwidth}
        \centering
        \includegraphics[scale=\scaleVar]{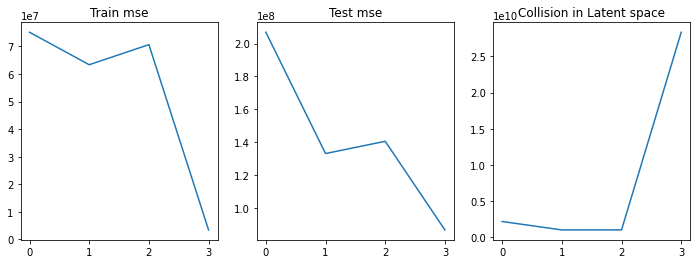}
        \caption{Water Converter}
    \end{subfigure}
    \\
    \begin{subfigure}[b]{\widthVar\textwidth}
        \centering
        \hspace{-5mm}
        \includegraphics[scale=\scaleVar]{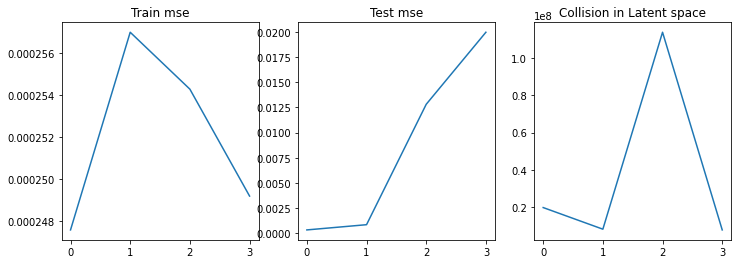}
        \caption{SPOKES}
    \end{subfigure}
    \caption{These curves show the network design test results. The collision value shown here is the penalty term proposed in \eqref{eq:penalty}. The x-axis denotes the neural network's general complexity. The collisions for model with lowest test MSE are still significant.}
    \label{fig:properNNs}
\end{figure*}

For \textbf{Max Area}, we test the three-layer dense neural network. The first layer consists of 50, 100, 1000 or 1500 neurons with Tanh activation functions. The second layer consists of 50 neurons with Tanh activation functions. The third layer consists of 10 neurons with Leaky Relu activation functions.
For \textbf{Rastrigin-2D}, we test the single-layer neural network, which consists of 10, 100, 1000, or 5000 neurons with Leaky Relu activation functions.
For \textbf{Water Converter}, we test the three-layer dense neural network. The first layer consists of 512, 1024, 2048, or 4096 neurons with Tanh activation functions. The second layer consists of half of the first layer's neurons with Tanh activation functions. The third layer consists of half of the second layer's neurons with Leaky Relu activation functions.
For \textbf{SPOKES}, we test the single-layer neural network, which consists of 10, 100, 1000, or 2000 neurons with Leaky Relu activation functions.

\section{Supplemental Materials on Algorithmic Details}\label{app:expsetup}
Our implementation of \algname and DW \algname is built upon the open source package GPytorch \citep{gardner2018gpytorch}. 
The deep kernel is trained with back propogation. We use the Adam \citep{kingma2014adam} optimizer with learning rate set to be $1e^{-2}$. Below we discuss the detailed configuration of the underlying neural network and the choice of the key parameters used by the algorithm.

\subsection{Algorithmic Details on Neural Network Architecture}
As the primary goal of our paper was to showcase the performance of a novel collision-free regularizer, we pick our network architectures to be basic multi-layer dense neural network (as illustrated by \figref{fig:multilayer-perceptron}). We use a 4-layer dense neural network. Its hidden layers consist of 1000, 500, 50 neurons respectively, each with Leaky Relu activation functions.  The output layer also uses Leaky Relu as its activation function and generates a 1-dimensional output.

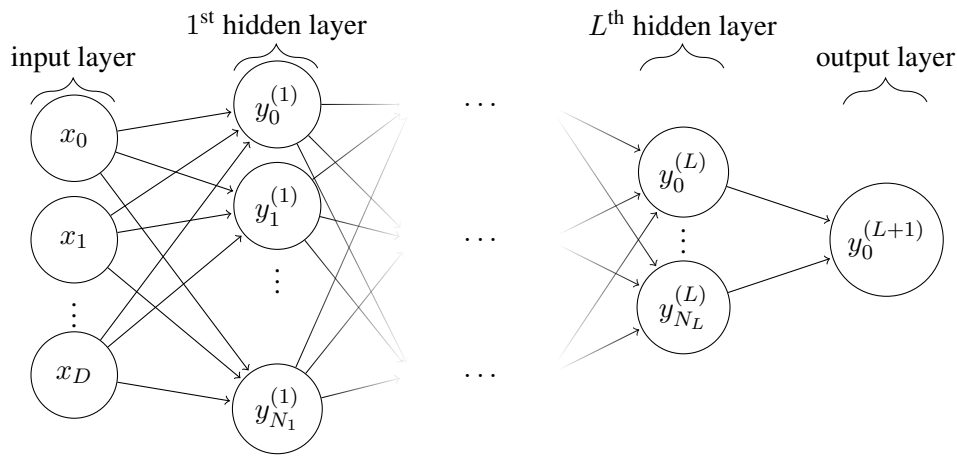
\begin{figure}[!h]
	\centering
	\begin{tikzpicture}[shorten >=1pt, scale=.9]
		\tikzstyle{unit}=[draw,shape=circle,minimum size=1.15cm]
		\tikzstyle{hidden}=[draw,shape=circle,minimum size=1.15cm]

		\node[unit](x0) at (0,3.5){$x_0$};
		\node[unit](x1) at (0,2){$x_1$};
		\node at (0,1){\vdots};
		\node[unit](xd) at (0,0){$x_D$};

		\node[hidden](h10) at (3,4){$y_0^{(1)}$};
		\node[hidden](h11) at (3,2.5){$y_1^{(1)}$};
		\node at (3,1.5){\vdots};
		\node[hidden](h1m) at (3,-0.5){$y_{N_1}^{(1)}$};

		\node(h22) at (5,0){};
		\node(h21) at (5,2){};
		\node(h20) at (5,4){};

		\node(d3) at (6,0){$\ldots$};
		\node(d2) at (6,2){$\ldots$};
		\node(d1) at (6,4){$\ldots$};

		\node(hL12) at (7,0){};
		\node(hL11) at (7,2){};
		\node(hL10) at (7,4){};

		\node[hidden](hL0) at (9,3){$y_0^{(L)}$};
		\node at (9,2.1){\vdots};
		\node[hidden](hL1) at (9,1){$y_{N_L}^{(L)}$};

		\node[unit](y1) at (12,2){$y_0^{(L+1)}$};

        \draw[->] (x0) -- (h10);
		\draw[->] (x0) -- (h11);
		\draw[->] (x0) -- (h1m);

        \draw[->] (x1) -- (h10);
		\draw[->] (x1) -- (h11);
		\draw[->] (x1) -- (h1m);

        \draw[->] (xd) -- (h10);
		\draw[->] (xd) -- (h11);
		\draw[->] (xd) -- (h1m);

		\draw[->] (hL0) -- (y1);
		\draw[->] (hL1) -- (y1);

        \draw[->,path fading=east] (h10) -- (h20);
		\draw[->,path fading=east] (h10) -- (h21);
		\draw[->,path fading=east] (h10) -- (h22);

		\draw[->,path fading=east] (h11) -- (h20);
		\draw[->,path fading=east] (h11) -- (h21);
		\draw[->,path fading=east] (h11) -- (h22);

		\draw[->,path fading=east] (h1m) -- (h20);
		\draw[->,path fading=east] (h1m) -- (h21);
		\draw[->,path fading=east] (h1m) -- (h22);

		\draw[->,path fading=west] (hL10) -- (hL1);
		\draw[->,path fading=west] (hL11) -- (hL1);
		\draw[->,path fading=west] (hL12) -- (hL1);
		\draw[->,path fading=west] (hL10) -- (hL0);
		\draw[->,path fading=west] (hL11) -- (hL0);
		\draw[->,path fading=west] (hL12) -- (hL0);

		\draw [decorate,decoration={brace,amplitude=10pt},xshift=-4pt,yshift=0pt] (-0.5,4) -- (0.75,4) node [black,midway,yshift=+0.6cm]{input layer};
		\draw [decorate,decoration={brace,amplitude=10pt},xshift=-4pt,yshift=0pt] (2.5,4.5) -- (3.75,4.5) node [black,midway,yshift=+0.6cm]{$1^{\text{st}}$ hidden layer};
		\draw [decorate,decoration={brace,amplitude=10pt},xshift=-4pt,yshift=0pt] (8.5,4.5) -- (9.75,4.5) node [black,midway,yshift=+0.6cm]{$L^{\text{th}}$ hidden layer};
		\draw [decorate,decoration={brace,amplitude=10pt},xshift=-4pt,yshift=0pt] (11.5,4) -- (12.75,4) node [black,midway,yshift=+0.6cm]{output layer};
	\end{tikzpicture}
	\caption[Network graph for a $(L+1)$-layer dense network.]{Network graph of a $(L+1)$-layer dense network with $D$ input units and $1$ output units. In our experiments, $L$ is set to be 3. Here $y_i^l$ denotes the $i^{th}$ neuron in the $l^{th}$ hidden layer.
	}
	\label{fig:multilayer-perceptron}
\end{figure}

\paragraph{Pre-training of the Neural Network}
We use the unlabeled dataset to pre-train an Auto-Encoder and use the parameters of its encoder to initialize the neural network following the protocol described by \citet{ferreira2020using}. Specifically, the encoder shares the same architecture with the neural network we described. Without the pre-training stage, the latent embedding fed by the neural network to the Gaussian process would be random. The practical problem with such randomness could be a much larger variance for the results since it influences the following neural network training process and the optimization process. 

 
\subsection{Parameter choices}
We investigate the robustness of parameter choices of the regularization parameter \rebuttal{$\regParam$} and $\regWeight$ on the Rastrigin 2D dataset. We show the results in \figref{fig:hyperparameter}.

\begin{figure}[!h]
     \centering
     \begin{subfigure}[b]{0.45\textwidth}
         \centering
        \includegraphics[width=\textwidth]{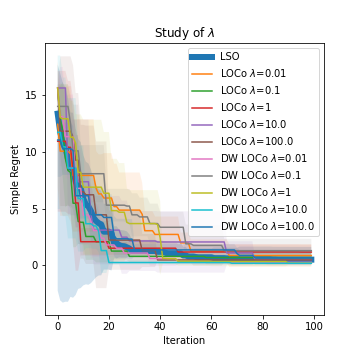}
         \caption{}\label{fig:lambda_study}
     \end{subfigure}
     \begin{subfigure}[b]{0.45\textwidth}
         \centering
        \includegraphics[width=\textwidth]{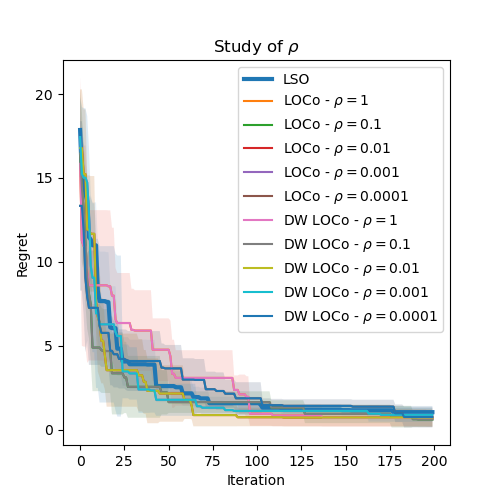}
         \caption{}
         \label{fig:rho_study}
     \end{subfigure}
     \caption{Simple regret under different parameter settings on the Rastrigin 2D dataset. The colored area represents the standard error of the tests at certain iteration. Experiments are repeated eight times for \selffigref{fig:lambda_study} and five times for \selffigref{fig:rho_study}. \rebuttal{\Figref{fig:lambda_study} shows that a moderately large $\lambda$ suffices to achieve decent performance in terms of simple regret. We believe that the wide range of objective values of the test dataset, which otherwise would hurt the optimization performance, can be regularized by the collision penalty.} \Figref{fig:rho_study} shows that a moderately large $\rho$ suffices to achieve decent performance in terms of simple regret. The curves demonstrate the robustness of \algname and DW \algname as long as the parameters are not set to be too extreme.
     }\label{fig:hyperparameter}
\end{figure}


\section{Additional Results}\label{app:additional_results}
\textit{Random EMbedding Bayesian Optimization} (REMBO) \citep{wang2016bayesian} leverages simple random linear transformations to improve the efficiency in low-effective-dimension high-dimensional tasks. We compare \algname and DW \algname with the performance of this random-embedding-based method and empirically exposed the failure case of REMBO when its modeling assumption does not hold (i.e. when dealing with dataset that has large effective dimensions). The results are summarized in \figref{fig:rembo_cmp}.

\begin{center}
     \centering
         \includegraphics[width=\textwidth]{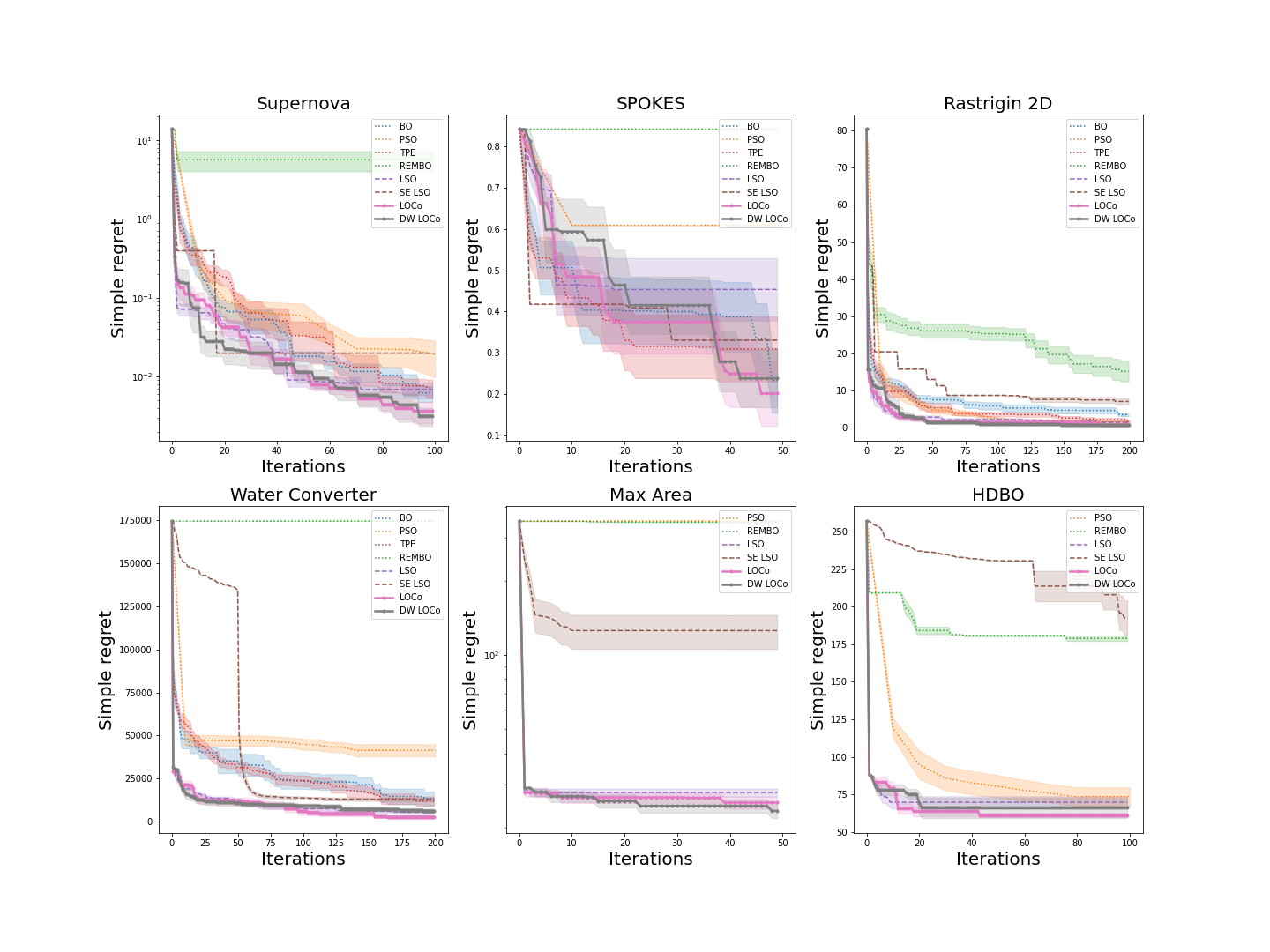}

     \captionof{figure}{Experiment results on six 
     synthetic \& real datasets. Each experiment is repeated at least eight times. The shaded area around the mean curve denotes the standard error.
     As illustrated in the figure, the random-embedding-based methods have been significantly outperformed by \algname and DW \algname. We place the discussion over REMBO here for two reasons. Firstly, there has been several problems about REMBO as discussed in \secref{sec:related}. Secondly, \rebuttal{the experiments are conducted on tasks where the effective dimensions are at a similar scale as the dimensionality of the original inputs} and doesn't align with the assumption of REMBO. 
}
\end{center} \label{fig:rembo_cmp}

\end{appendix}

\end{document}